\documentclass{article}

\usepackage{times}
\usepackage{graphicx} 
\usepackage{caption}
\usepackage{subcaption}

\usepackage{natbib}

\usepackage{algorithm}
\usepackage{algorithmic}

\usepackage{hyperref}

\usepackage[accepted]{icml2017}

\icmltitlerunning{Approximate Steepest Coordinate Descent (ASCD)}

\graphicspath{ {pics/} }
\usepackage{paralist}
\usepackage{booktabs} 

\usepackage{amsmath,amssymb,xspace,amsthm,Tabbing,bm}
\usepackage{color,mathtools}
\usepackage{listings}
\usepackage{multirow}
\usepackage{bbm}

\usepackage{wrapfig}
\usepackage{tikz}
\usepackage{wasysym}

\def\argmax{\mathop{\rm arg\,max}}
\def\argmin{\mathop{\rm arg\,min}}

\newcommand{\reals}{\mathbb{R}}

\def\argmax{\mathop{\rm arg\,max}}
\def\argmin{\mathop{\rm arg\,min}}

\newcommand{\vv}{\mathbf{v}}

\newcommand{\bq}{\begin{equation}}
\newcommand{\eq}{\end{equation}}
\newcommand{\ba}{\begin{eqnarray}}
\newcommand{\ea}{\end{eqnarray}}

\def\R{{\reals}}

\newcommand{\remove}[1]{}

\newcommand{\norm}[1]{\left\lVert #1\right\rVert}
\newcommand{\abs}[1]{\left\lvert #1\right\rvert}

\newcommand{\xv}{\bm{x}}
\newcommand{\yv}{\bm{y}}

\newcommand{\av}{\bm{a}}
\newcommand{\bv}{\bm{b}}

\newcommand{\evv}{\bm{e}}

\newcommand{\uv}{\bm{u}}

\newcommand{\wv}{\bm{w}}
\newcommand{\omegav}{\bm{\omega}}

\newcommand{\alphav}{\bm{\alpha}}

\newcommand{\E}[1]{\mathbb{E}\left[#1\right] } 
\newcommand{\EE}[2]{\mathbb{E}_{#1}\left[#2\right] }

\newtheorem{definition}{Definition}[section]
\newtheorem{theorem}{Theorem}[section]

\newtheorem{lemma}[theorem]{Lemma}

\newcommand{\refLE}[1]{\ensuremath{\stackrel{(\ref{#1})}{\leq}}}
\newcommand{\refEQ}[1]{\ensuremath{\stackrel{(\ref{#1})}{=}}}

\newcommand{\dotprod}[1]{\langle #1\rangle}
\newcommand{\gv}{\bm{g}}
\newcommand{\ellv}{\bm{\ell}}
\newcommand{\gtv}{\bm{\tilde{g}}}
\newcommand{\rv}{\bm{r}}
\newcommand{\usv}{\bm{u}^\star}
\newcommand{\ellsv}{\bm{\ell}^\star}

\newenvironment{CompactEnumerate}{
\begin{list}{\arabic{enumi}.}{%
\usecounter{enumi}
\setlength{\leftmargin}{12pt}
\setlength{\itemindent}{2pt}
\setlength{\topsep}{0pt}
\setlength{\itemsep}{-2pt}
}}

\begin{document}

\twocolumn[
\icmltitle{Approximate Steepest Coordinate Descent}

\icmlsetsymbol{equal}{*}

\begin{icmlauthorlist}
\icmlauthor{Sebastian U. Stich}{epfl}
\icmlauthor{Anant Raj}{mpi}
\icmlauthor{Martin Jaggi}{epfl}
\end{icmlauthorlist}

\icmlaffiliation{epfl}{EPFL}
\icmlaffiliation{mpi}{Max Planck Institute for Intelligent Systems}

\icmlcorrespondingauthor{Sebastian U. Stich}{sebastian.stich@epfl.ch}

\vskip 0.3in
]

\printAffiliationsAndNotice{}  

\begin{abstract} 
We propose a new selection rule for the coordinate selection in coordinate descent methods for huge-scale optimization. The efficiency of this novel scheme is provably better than the efficiency of uniformly random selection, and can reach the efficiency of steepest coordinate descent (SCD), enabling an acceleration of a factor of up to $n$, the number of coordinates. In many practical applications, our scheme can be implemented at no extra cost and computational efficiency very close to the faster uniform selection.
Numerical experiments with Lasso and Ridge regression show promising improvements, in line with our theoretical guarantees.
\end{abstract}

\section{Introduction}
\label{sec:intro}
Coordinate descent (CD) methods have attracted a substantial interest the optimization community in the last few years~\cite{nesterov2012,richtarik:2016}.
Due to their
computational efficiency, scalability, as well as their ease of implementation, 
these methods are the state-of-the-art for a wide selection of 
machine learning and 
signal processing applications \citep{Fu:1998cd,
Hsieh:2008bd,
wright2015
}. This is also theoretically well justified:
The complexity estimates for CD methods are in general better than the estimates for methods that compute the full gradient in one batch pass~\cite{nesterov2012,Nesterov:2017}.

In many CD methods, the active coordinate is picked at random, according to a probability distribution.
For smooth functions it is theoretically well understood how the sampling procedure is related to the efficiency of the scheme and which distributions give the best complexity estimates~\cite{nesterov2012,Zhaoa15,AllenYhu:2016,zheng:2016,Nesterov:2017}. %
For nonsmooth and composite functions --- that appear in many machine learning applications --- the picture is less clear. 
For instance in \cite{ShalevShwartz:2013wl,Friedman:2007ut,Friedman:2010wm,ShalevShwartz:2011vo} uniform sampling (UCD) is used, whereas other papers propose adaptive sampling strategies that change over time~\cite{Papa2015,Csiba:2015ue,Osokin:2016:MGB,Perekrestenko17a}.

A very simple deterministic strategy is to move along the direction corresponding to the component of the gradient with the maximal absolute value (steepest coordinate descent, SCD)~\cite{Boyd:2004uz,Tseng2009}. For smooth functions this strategy yields always better progress than UCD, and the speedup can reach a factor of the dimension~\cite{nutini2015coordinate}. However, SCD requires the computation of the whole gradient vector in each iteration which is prohibitive (except for special applications, cf.~\citet{Dhillon:2011uh,Shrivastava:2014:ALS}). %

In this paper we propose approximate steepest coordinate descent (ASCD), a novel scheme which combines the best parts of the aforementioned strategies: (i) ASCD maintains an approximation of the \emph{full} gradient in each iteration and selects the active coordinate 
among the components of this vector that have large absolute values --- similar to SCD; and (ii) in many situations the gradient approximation can be updated cheaply at no extra cost --- similar to %
UCD.
We show that regardless of the errors in the gradient approximation (even if they are infinite), ASCD performs always better than UCD.

Similar to the methods proposed 
in~\cite{Tseng2009} we also present variants of ASCD for composite problems. We confirm our theoretical findings by numerical experiments for Lasso and Ridge regression on a synthetic dataset as well as on the 
RCV1 (binary) dataset.

\vspace{-5pt}
\paragraph{Structure of the Paper and Contributions.}
In Sec.~\ref{sec:steep} we review the existing theory for SCD and (i) extend it to the setting of smooth functions.
We present (ii) a novel lower bound, showing that the complexity estimates for SCD and UCD can be equal in general. We (iii) introduce ASCD and the save selection rules for both smooth (Sec.~\ref{sec:algorithm}) and to composite functions (Sec.~\ref{sec:composite}). 
We prove that (iv) ASCD performs always better than UCD (Sec.~\ref{sec:algorithm}) and (v)
it can reach the performance of SCD (Sec.~\ref{sec:competitiveRatio}).
In Sec.~\ref{sec:gradUpdates} we discuss important applications where the gradient estimate can efficiently be maintained.
Our theory is supported by numerical evidence in Sec.~\ref{sec:emp_eval}, which reveals that (vi) ASCD performs extremely well on real data.

\paragraph{Notation.}
Define $[\xv]_i := \dotprod{\xv, \evv_i}$ with
$\evv_i$ the standard unit vectors in $\R^n$.
We abbreviate $\nabla_i f := [\nabla f]_i$.
A convex function $f \colon \R^n \to \R$ with coordinate-wise $L_i$-Lipschitz continuous gradients\footnote{%
$\abs{ \nabla_i f(\xv + \eta \evv_i) - \nabla_i f(\xv)} \leq L_i \abs{\eta},\quad \forall \xv \in \R^n, \eta \in \R$.
} for constants $L_i > 0$, $i \in [n] := \{1,\dots,n\}$, satisfies by the standard reasoning 
\begin{align}
f(\xv + \eta \evv_i) \leq f(\xv) + \eta \nabla_i f(\xv)  + \tfrac{L_i}{2} \eta^2  \label{eq-Ubound}
\end{align}
for all $\xv \in \R^n$ and $\eta \in \R$. A function is coordinate-wise $L$-smooth if $L_i \leq L$ for $i=1,\dots,n$.
For an optimization problem $\min_{\xv \in \R^n} f(\xv)$ define $X^\star := \argmin_{\xv \in \R^n} f(\xv)$ and denote by $\xv^\star \in \R^n$ an arbitrary element $\xv^\star \in X^\star$. 
\section{Steepest Coordinate Descent} \label{sec:steep}
In this section we present SCD %
and discuss its theoretical properties. The functions of interest are composite convex functions $F \colon \R^n \to \R$ of the form
\begin{align}
 F(\xv) := f(\xv) + \Psi(\xv) \label{def-composite}
\end{align}
where $f$ is coordinate-wise $L$-smooth and $\Psi$ convex and separable, that is that is $\Psi(\xv) = \sum_{i=1}^n \Psi_i([\xv]_i)$. 
In the first part of this section we focus on smooth problems, i.e. we assume that $\Psi \equiv 0$. %

Coordinate descent methods with constant step size generate a sequence $\{\xv_t\}_{t \geq 0}$ of iterates that satisfy the relation
\begin{align}
\xv_{t+1} = \xv_t - \tfrac{1}{L} \nabla_{i_t} f(\xv) \evv_{i_t} \,. \label{eq:grad_update}
\end{align}
In UCD the active coordinate $i_t$ is chosen uniformly at random from the set $[n]$, $i_t \in_{u.a.r.} [n]$. SCD chooses the coordinate according to the Gauss-Southwell (GS) rule:
\begin{align}
 i_{t} = \argmax_{i \in [n]} \nabla_{i} \abs{ f(\xv_t) } \,. \label{eq:scgd}
\end{align}

\subsection{Convergence analysis}
With the quadratic upper bound~\eqref{eq-Ubound} one can easily get a lower bound on the one step progress
\begin{align}
 \E{f(x_t) - f(x_{t+1}) \mid x_t} \geq \EE{i_t} { \tfrac{1}{2L} \abs{\nabla_{i_t} f(\xv_t)}^2}  \,. \label{eq-quadONESTEP}
\end{align}
For UCD and SCD the expression on the right hand side evaluates to
\begin{align}
\begin{split}
 \tau_{\rm UCD}(\xv_t) &:= \tfrac{1}{2nL} \norm{\nabla f(\xv_t)}_2^2 \\
 \tau_{\rm SCD}(\xv_t) &:= \tfrac{1}{2L} \norm{\nabla f(\xv_t)}_\infty^2  \label{def-ONESTEP}
 \end{split}
\end{align}
With Cauchy-Schwarz we find
\begin{align}
  \tfrac{1}{n}  \tau_{\rm SCD}(\xv_t) \leq  \tau_{\rm UCD}(\xv_t)  \leq  \tau_{\rm SCD}(\xv_t) \,. \label{eq-compareSU}
\end{align}
Hence, the lower bound on the one step progress of SCD is always at least as large as the lower bound on the one step progress of UCD. Moreover, the one step progress could be even lager by a factor of $n$. However, it is very difficult to formally prove that this linear speed-up holds  for more than one iteration, as the expressions in~\eqref{eq-compareSU} depend on the (a priori unknown) sequence of iterates $\{\xv_t\}_{t \geq 0}$.

\paragraph{Strongly Convex Objectives.}
\citet{nutini2015coordinate} present an elegant solution of this problem for $\mu_2$-strongly convex functions%
\footnote{A function is $\mu_p$-strongly convex in the $p$-norm, $p\geq 1$, if $f(\yv) \geq f(\xv) + \dotprod{\nabla f(\xv), \yv - \xv} + \frac{\mu_p}{2} \norm{\yv - \xv}_p^2$, $\forall \xv,\yv \in \R^n$.}. 
They propose to measure the strong convexity of the objective function in the $1$-norm instead of the $2$-norm. This gives rise to the lower bound
\begin{align}
 \tau_{\rm SCD}(\xv_t) \geq \tfrac{\mu_1}{L} \left(f(\xv_t) - f(\xv^\star) \right) \,,
\end{align}
where $\mu_1$ denotes the strong convexity parameter. By this, they get a uniform upper bound on the convergence that does not directly depend on local properties of the function, like for instance $\tau_{\rm SCD}(\xv_t)$, but just on $\mu_1$. It always holds $\mu_1 \leq \mu_2$, and for functions where both quantities are equal, SCD enjoys a linear speedup over UCD. %

\paragraph{Smooth Objectives.}
When the objective function $f$ is just smooth (but not necessarily strongly convex), then the analysis mentioned above is not applicable. We here extend the analysis from~\cite{nutini2015coordinate} to smooth functions.

\begin{theorem} \label{thm:smooth_steep_converge}
Let $f \colon \R^n \to \R$ be convex and coordinate-wise $L$-smooth. Then for the sequence $\{\xv_t\}_{t \geq 0}$ generated by SCD it holds:
\begin{align} \label{eq:smooth_converge_steep}
f(\xv_t) - f(\xv^\star) \leq \frac{2L R_1^2}{t}\,,
\end{align}
for $\displaystyle R_1 := \max_{\xv^\star \in X^\star} \left\{ \max_{\xv \in \R^n} \left[  \norm{\xv - \xv^\star}_1 \mid f(\xv) \leq f(\xv_0)  \right] \right\}$.
\end{theorem}
\begin{proof}
In the proof we first derive a lower bound on the one step progress (Lemma~\ref{lem:smooth_steepest}), similar to the analysis in~\cite{nesterov2012}.
The lower bound for the one step progress
of SCD can in each iteration differ up to a factor of $n$ from the analogous bound derived for UCD (similar as in~\eqref{eq-compareSU}). All details are given in Section~\ref{sec-proofs-for-smooth} in the appendix.
\end{proof}

Note that the $R_1$ is essentially the diameter of the level set at $f(\xv_0)$ measured in the $1$-norm. In the complexity estimate of UCD, $R_1^2$ in~\eqref{eq:smooth_converge_steep} is replaced by $nR_2^2$, where $R_2$ is the diameter of the level at $f(\xv_0)$ measured in the $2$-norm (cf.~\citet{nesterov2012,wright2015}). As in~\eqref{eq-compareSU} we observe with Cauchy-Schwarz
\begin{align}
\tfrac{1}{n} R_1^2 \leq R_2^2 \leq  R_1^2\,,
\end{align}
i.e. SCD can accelerate  up to a factor of $n$ over to UCD.

\subsection{Lower bounds}
\label{sec-lowerbound}
In the previous section we provided complexity estimates for the methods SCD and UCD and showed that SCD can converge up to a factor of the dimension $n$ faster than UCD. In this section we show that this analysis is tight. In Theorem~\ref{thm-lowerbound} below we give a function  $q \colon \R^n \to \R$, for which the one step progress  $\tau_{\rm SCD}(\xv_t) \approx  \tau_{\rm UCD}(\xv_t)$ up to a constant factor, for all iterates $\{\xv_t\}_{t \geq 0}$ generated by SCD. 

By a simple technique we can also construct functions for which the speedup is exactly equal to an arbitrary factor $\lambda \in [1,n]$. For instance we can consider functions with a (separable) low dimensional structure. Fix integers $s,n$ such that $\frac{n}{s} \approx \lambda$, define the function  $f \colon \R^n \to \R$ as
\begin{align}
f(\xv) := q (\pi_s (\xv)) \label{def-lowdim}
\end{align}
where $\pi_s$ denotes the projection to $\R^s$ (being the first $s$ out of $n$ coordinates) and $q \colon \R^s \to \R$ is the function from Theorem~\ref{thm-lowerbound}. Then
\begin{align}
\tau_{\rm SCD}(\xv_t) \approx {\lambda} \cdot \tau_{\rm UCD}(\xv_t)\,,
\end{align}
for all iterates $\{\xv_t\}_{t \geq 0}$ generated by SCD.

\begin{theorem}
\label{thm-lowerbound}
Consider the function $q(\xv) = \frac{1}{2} \dotprod{Q\xv,\xv}$ for $Q := I_n - \frac{99}{100n} J_n$, where $J_n = 1_n1_n^T$, $n>2$. Then there exists $\xv_0 \in R^n$ such that for the sequence $\{\xv_t\}_{t \geq 0}$ generated by SCD it holds
\begin{align}
 \norm{ \nabla q(\xv_t) }_\infty^2 \leq \tfrac{4}{n} \norm{ \nabla q(\xv_t) }_2^2\,. \label{eq-lowerboundratio}
\end{align}
\end{theorem}
\begin{proof}
In the appendix we discuss a family of functions defined by matrices $Q := (\alpha -1)\frac{1}{n} J_n + I_n$ and define corresponding parameters $0 < c_\alpha < 1$ such that for $\xv_0$ defined as $[\xv_0]_i = c_\alpha^{i-1}$ for $i=1,\dots,n$, SCD cycles through the coordinates, that is, the sequence $\{\xv_t\}_{t \geq 0}$ generated by SCD satisfies 
\begin{align}
[\xv_t]_{1 + (t-1 \mod n)} = c_\alpha^n \cdot [\xv_{t-1}]_{1 + (t-1 \mod n)}\,. 
\end{align}
We verify that for this sequence property~\eqref{eq-lowerboundratio} holds.
\end{proof}

\subsection{Composite Functions}
\label{sec-compositeintro}
The generalization of the GS rule~\eqref{eq:scgd} to composite problems~\eqref{def-composite} with nontrival $\Psi$ is not straight forward. The `steepest' direction is not always meaningful in this setting; consider for instance a constrained problem where this rule could yield no progress at all when stuck at the boundary.

\citet{nutini2015coordinate} discuss several generalizations of the Gauss-Southwell rule for composite functions.
The GS-s rule is defined to choose the coordinate
with the most negative directional derivative \cite{wu2008}. This rule is identical to~\eqref{eq:scgd} but requires the calculation of subgradients of $\Psi_i$. 
However, the length of a step could be arbitrarily small. In contrast, the GS-r rule was defined to pick the coordinate direction that yields the longest step \cite{Tseng2009}. The rule that enjoys the best theoretical properties (cf. \citet{nutini2015coordinate}) is the GS-q rule, which is defined as to maximize the progress assuming a quadratic upper bound on $f$ \cite{Tseng2009}. Consider the coordinate-wise models
\begin{align}
 V_i(\xv,y,s) := sy + \tfrac{L}{2}y^2 + \Psi_i([\xv]_i+y) \,, \label{def-V}
\end{align}
for $i\in [n]$. The GS-q rule is formally defined as
\begin{align}
 i_{\rm GS-q} = \argmin_{i \in [n]} \min_{y \in \R} V_i(\xv,y, \nabla_i f(\xv)) \,. \label{def-GSq}
\end{align}

\subsection{The Complexity of the GS rule}
\label{sec-complexity}
So far we only studied the iteration complexity of SCD, but we have disregarded the fact that the computation of the GS rule~\eqref{eq:scgd} can be as expensive as the computation of the whole gradient. The application of coordinate descent methods is only justified if the complexity to compute one directional derivative is approximately $n$ times cheaper than the computation of the full gradient vector (cf.~\citet{nesterov2012}). By Theorem~\ref{thm-lowerbound} this reasoning also applies to SCD. A class of function with this property is given by functions $F \colon \R^n \to  \R$
\begin{align}
 F(\xv) :=  f(A\xv) +  \sum_{i=1}^n \Psi_i([\xv]_i) \label{eq:primalS}
\end{align}
where $A$ is a $d \times n$ matrix, and where $f \colon \R^d \to \R$, and $\Psi_i \colon \R^n \to \R$ are convex and simple, that is the time complexity $T$ for computing their gradients is linear: $T(\nabla_{\yv} f(\yv),\nabla_{\xv} \Psi(\xv) = O(d + n)$. 
This class of functions includes least squares, logistic regression, Lasso, and SVMs (when solved in dual form).

Assuming the matrix is dense, the complexity to compute the full gradient of $F$ is $T(\nabla_{\xv}F(\xv)) = O(dn)$. If the value $\wv=A\xv$ is already computed, one directional derivative can be computed in time $T(\nabla_{i} F(\xv)) = O(d)$. The recursive update of $\wv$ after one step needs the addition of one column of matrix $A$ with some factors and can be done in time $O(d)$. However, we note that recursively updating the full gradient vector takes time $O(dn)$ and consequently the computation of the GS rule \emph{cannot} be done efficiently.

\citet{nutini2015coordinate} consider sparse matrices, for which the computation of the Gauss-Southwell rule becomes traceable. In this paper, we propose an alternative approach. Instead of updating the exact gradient vector, we keep track of an approximation of the gradient vector and recursively update this approximation in time $O(n \log n)$. With these updates, the use of coordinate descent is still justified in case $d = \Omega(n)$.

\section{Algorithm}
\label{sec:algorithm}
Is it possible to get the significantly improved convergence speed from SCD, when one is only willing to pay the computational cost of only the much simpler UCD? In this section, we give a formal definition of our proposed approximate SCD method which we denote ASCD. 

The core idea of the algorithm is the following: While performing coordinate updates, ideally we would like to efficiently track the evolution of \emph{all} elements of the gradient, not only the one coordinate which is updated in the current step.
The formal definition of the method is given in Algorithm~\ref{alg-1} for smooth objective functions. 
In each iteration, only one coordinate is modified according to some arbitrary update rule $\mathcal{M}$.
The coordinate update rule $\mathcal{M}$ provides two things: First the new iterate $\xv_{t+1}$, and secondly also an estimate $\tilde g$ of the $i_t$-th entry of the gradient at the new iterate%
\footnote{For instance, for updates by exact coordinate optimization (line-search), we have $\tilde g=r=0$.}. 
Formally,
\begin{align}
(\xv_{t+1}, \tilde g, r) := \mathcal{M}(\xv_t, \nabla_{i_t} f(\xv_t)) \label{def-M}
\end{align}
such that the quality of the new gradient estimate $\tilde g$ satisfies 
\begin{align}
\abs{\nabla_{i_t} f(\xv_{t+1}) - \tilde g} \leq r\,.  \label{eq-errorbound}
\end{align}
The non-active coordinates are updated with the help of gradient oracles with accuracy $\delta \geq 0$ (see next subsection for details). 
The scenario of exact updates of all gradient entries is obtained for accuracy parameters $\delta=r=0$ and in this case ASCD is identical to SCD.

\begin{algorithm}[tb]
   \caption{Approximate SCD (ASCD)}
   \label{alg-1}
\renewcommand{\algorithmiccomment}[1]{\hfill {\footnotesize\textit{#1}} \null}   
\begin{algorithmic}
   \STATE {\bfseries Input:} $f$, $\xv_0$, $T$, $\delta$-gradient oracle $g$, method $\mathcal{M}$
   \STATE Initialize $[\gtv_0]_i = 0$, $[\rv_0]_i = \infty$ for $i\in[n]$. 
   \FOR{$t=0$ {\bfseries to} $T$}
   \STATE For $i \in [n]$ define \COMMENT{compute u.-and l.-bounds}
   \STATE $[\uv_t]_i := \max \{ \abs{ [\gtv_t]_i - [\rv_t]_i } , \abs{ [\gtv_t]_i + [\rv_t]_i} \}$
   \STATE $[\ellv_t]_i := \min_{y \in \R } \{\abs{y} \mid [\gtv_t]_i - [\rv_t]_i \leq y \leq [\gtv_t]_i + [\rv_t]_i \}$\\[1ex]
   \STATE $\operatorname{av}({\mathcal{I}}) := \frac{1}{\abs{\mathcal{I}}} \sum_{i \in \mathcal{I}} [\ellv_t]_i^2$ \COMMENT{compute active set}
   \STATE $\mathcal{I}_t := \argmin_{\mathcal{I}} \abs{\left\{\mathcal{I} \subseteq [n] \mid [\uv_{t}]_i^2 < \operatorname{av}({\mathcal{I}}), \forall i \notin \mathcal{I}\right\}}$\\[1ex]
   \STATE Pick $i_t \in_{\rm u.a.r.} \argmax_{i \in \mathcal{I}_t} \{[\ellv]_i\}$ \COMMENT{active coordinate}
   \STATE  \hfill $(\xv_{t+1},[\gtv_{t+1}]_{i_t},[\rv_{t+1}]_{i_t}) := \mathcal{M}(\xv_t,\nabla_{i_t}f(\xv_t))$ \hfill\null \\[1ex]
   \STATE $\gamma_t := [\xv_{t+1}]_{i_t} - [\xv_t]_{i_t}$ \COMMENT{update $\nabla f(\xv_{t+1})$ estimate}
   \STATE Update $[\gtv_{t+1}]_j := [\gtv_{t}]_j + \gamma_t g_{i_t j}(\xv_t)$, $j\neq i_t$
   \STATE Update $[\rv_{t+1}]_{j} := [\rv_{t}]_{j} + \gamma_t \delta_{i_t j}$, $j\neq i_t$
   \ENDFOR
\end{algorithmic}
\end{algorithm}

\subsection{Safe bounds for gradient evolution}
ASCD maintains lower and upper bounds for the absolute values of each component of the gradient 
($[\ellv]_i \leq \abs{\nabla_i f(\xv)} \leq [\uv]_i$).
These bounds allow to identify the coordinates on which the absolute values of the gradient are small  (and hence cannot be the steepest one).
More precisely, the algorithm maintains a set $\mathcal{I}_t$ of active coordinates (similar in spirit as in active set methods, see e.g. \citet{Kim:2008,wen:2012}). A coordinate $j$ is excluded from $\mathcal{I}_t$ if the estimated progress in this direction (cf.~\eqref{eq-quadONESTEP}) is lower than the average of the estimated progress along coordinate directions in $\mathcal{I}_t$, $[\uv_t]_j^2 < \frac{1}{\abs{\mathcal{I}_t}}\sum_{i \in \mathcal{I}_t} [\ellv_t]_i^2$. The active set $\mathcal{I}_t$ %
can be computed in $O(n \log n)$ time by sorting. All other operations take linear $O(n)$ time.

\paragraph{Gradient Oracle.}
The selection mechanism in ASCD 
crucially relies on the following definition of a $\delta$-gradient oracle. While the update $\mathcal{M}$ delivers the estimated active entry of the new gradient, the additional gradient oracle is used to update all other coordinates $j \ne i_t$ of the gradient; as in the last two lines of Algorithm~\ref{alg-1}.

\begin{definition}[$\delta$-gradient oracle]
For a function $f \colon \R^n \to \R$ and indices $i,j \in [n]$, a $(i,j)$-gradient oracle with error $\delta_{ij} \geq 0$ is a function $g_{ij}\colon \R^n \to \R$ satisfying $\forall \xv \in \R^n, \forall \gamma \in \R$:
\begin{align}
\abs{ \nabla_j f(\xv + \gamma \evv_i) - \gamma g_{ij}(\xv) } &\leq \abs{\gamma} \delta_{ij}\,. 
\label{def-GO}
\end{align}
We denote by a $\delta$-gradient oracle a family $\{g_{ij}\}_{i,j \in [n]}$ of $\delta_{ij}$-gradient oracles.
\end{definition}

We discuss the availability of good gradient oracles for many problem classes in more detail in Section \ref{sec:gradUpdates}.
For example for least squares problems and general linear models
, a $\delta$-gradient oracle is for instance given by a scalar product estimator as in~\eqref{def-ApproxScalar} below. %
Note that ASCD can also handle very bad estimates, as long as the property~\eqref{def-GO} is satisfied (possibly even with accuracy $\delta_{ij} = \infty$). 

\paragraph{Initialization.}
In ASCD the initial estimate $\gtv_0$ of the gradient is just arbitrarily set to $\bm{0}$, with uncertainty $\rv_0 = \infty$. Hence in the worst case it takes $\Theta(n \log n)$ iterations  until each coordinate gets picked at least once (cf. \citet{Dawkins:1991}) and until corresponding gradient estimates are set to a realistic value. If better estimates of the initial gradient are known, they can be used for the initialization as long as a strong error bound as in~\eqref{eq-errorbound} is known as well. For instance the initialization can be done with $\nabla f(\xv_0)$ if one is willing to compute this vector in one batch pass.

\paragraph{Convergence Rate Guarantee.}
We present our first main result showing that the performance of ASCD is provably between UCD and SCD. First observe
that if in Algorithm~\ref{alg-1} the gradient oracle is always exact, i.e. $\delta_{ij} \equiv 0$, and if $\gtv_0$ is initialized with $\nabla f(\xv_0)$, then in each iteration $\abs{\nabla_{i_t} f(\xv_t)} = \norm{\nabla f(\xv_t)}_\infty$ and ASCD identical to SCD. 
\begin{lemma} \label{lem-contained}
Let $i_{\rm max} := \argmax_{i \in [n]} \abs{\nabla_i f(\xv_t)}$. Then $i_{\rm max} \in \mathcal{I}_t$, for $\mathcal{I}_t$ as in Algorithm~\ref{alg-1}.
\end{lemma}
\begin{proof}
This is immediate from the definitions of $\mathcal{I}_t$ and the upper and lower bounds. Suppose $i_{\rm max} \notin \mathcal{I}_t$, then there exists $j \neq i_{\rm max}$ such that $[\ell_t]_j > [u_t]_{i_{\rm max}}$, and consequently $\abs{\nabla_j f(\xv_t)} > \abs{\nabla_{i_{\rm max}} f(\xv_t)}$.
\end{proof}

\begin{theorem}
\label{thm-sandwich}
Let $f \colon \R^n \to \R$ be convex and coordinate-wise $L$-smooth, let $\tau_{\rm UCD},\tau_{\rm SCD},\tau_{\rm ASCD}$ denote the expected one step progress~\eqref{def-ONESTEP} of UCD, SCD and ASCD, respectively, 
and suppose all methods use the same step-size rule~$\mathcal{M}$. Then
\begin{align}
 \tau_{\rm UCD}(\xv) \leq \tau_{\rm ASCD}(\xv)\leq \tau_{\rm SCD}(\xv) \quad \forall \xv \in \R^n\,. \label{eq-sandwich}
\end{align}
\end{theorem}
\begin{proof}
By~\eqref{eq-quadONESTEP} we 
get $\tau_{\rm ASCD}(\xv) = \frac{1}{2L \abs{\mathcal{I}}} \sum_{i \in \mathcal{I}} \abs{\nabla_i f(\xv)}^2$, where $\mathcal{I}$ denotes the corresponding index set of ASCD when at iterate $\xv$. 
Note that for $j \notin \mathcal{I}$ it must hold that $\abs{\nabla_j f(\xv)}^2 \leq [\uv]_j^2 < \frac{1}{\abs{\mathcal{I}}} \sum_{i \in \mathcal{I}} [\ellv]_i^2 \leq  \frac{1}{ \abs{\mathcal{I}}} \sum_{i \in \mathcal{I}} \abs{\nabla_i f(\xv)}^2$ by definition of $\mathcal{I}$. 
\end{proof}

Observe that the above theorem holds for all gradient oracles and coordinate update variants, as long as they are used with corresponding quality parameters $r$ (as in \eqref{eq-errorbound}) and $\delta_{ij}$ (as in \eqref{def-GO}) as part of the algorithm.

\paragraph{Heuristic variants.}
Below also propose three heuristic variants of ASCD. For all these variants the active set $\mathcal{I}_t$
can be computed $O(n)$, but the statement of Theorem~\ref{thm-sandwich} does not apply. These variants only differ from ASCD in the choice of the active set in Algorithm~\ref{alg-1}:
\CompactEnumerate
\setlength{\itemindent}{15pt}
\item[u-ASCD:] $\mathcal{I}_t := \argmax_{i \in [n]} [\uv_t]_i$
\item[$\ell$-ASCD:] $\mathcal{I}_t := \argmax_{i \in [n]} [\ellv_t]_i$
\item[a-ASCD:] $\mathcal{I}_t := \left\{i \in [n] \mid [\uv_{t}]_i \geq \max_{i \in [n]} [\ellv_{t}]_i \right\}$
\end{list}

\section{Approximate Gradient Update}   
\label{sec:gradUpdates}
In this section we argue that for a large class of objective functions
of interest in machine learning, the change in the gradient along every coordinate direction can be estimated efficiently. 

\begin{lemma} \label{lem:gen_grad_update}
Consider $F \colon \R^n \to \R$ as in~\eqref{eq:primalS}
with twice-differentiable $f \colon \R^d \to \R$. Then for two iterates $\xv_t, \xv_{t+1} \in \R^n$ of a coordinate descent algorithm, i.e.  $\xv_{t+1} = \xv_t 
+ \gamma_t \evv_{i_t}$, there exists a $\tilde{\xv} \in \R^n$ on the line segment between $\xv_t$ and $\xv_{t+1}$, $\tilde{\xv} \in [\xv_t, \xv_{t+1}] $ with
\begin{align}
\nabla_i F(\xv_{t+1}) - \nabla_i F(\xv_t) = \gamma_t  \dotprod{  \av_i, \nabla^2 f(A \tilde{\xv} ) \av_{i_t}}  \quad \forall i \neq i_t \label{eq:approx_general}
\end{align} 
\vskip-20pt
where $\av_i$ denotes the $i$-th column of the matrix $A$. 
\end{lemma}
\begin{proof}
For coordinates $i \neq i_t$ the gradient (or subgradient set) of $\Psi_i([\xv_t]_i)$  does not change. Hence it suffices to calculate the change $\nabla f(\xv_{t+1}) - \nabla f(\xv_t)$. This is detailed in the appendix.
\end{proof}

\paragraph{Least-Squares with Arbitrary Regularizers.} The least squares problem is defined as problem~\eqref{eq:primalS} with $f(A\xv) = \frac{1}{2}\norm{A\xv - \bv }_2^2$ for a $\bv \in \R^d$. This function is twice differentiable with $\nabla^2 f(A\xv) = I_n$. Hence~\eqref{eq:approx_general} reduces to
\begin{align}
\nabla_i F(\xv_{t+1}) - \nabla_i F(\xv_t) = \gamma_t \dotprod{\av_i, \av_{i_t}}  \quad \forall i \neq i_t \,. \label{eq:exact_leastsq}
\end{align}

This formulation gives rise to various gradient oracles~\eqref{def-GO} for the least square problems. For  for $i \neq i_t$ we easily verify that the condition~\eqref{def-GO} is satisfied:
\CompactEnumerate
\item  $g^1_{ij} := \dotprod{\av_i, \av_{i_t}}$; $\delta_{ij} = 0$,
\item  $g^2_{ij} := \max\left\{-\norm{\av_i}\norm{\av_j}, \min\left\{S(i,j), \norm{\av_i}\norm{\av_j} \right\}\right\} $; $\delta_{ij} = \epsilon \norm{\av_i}\norm{\av_j}$, where $S \colon [n] \times [n]$ denotes a function with the property
\end{list}
\vspace{-12pt}
\begin{align}
\abs{ S(i,j) - \dotprod{\av_i,\av_j}} \leq \epsilon \norm{\av_i}\norm{\av_j}\,, \quad \forall i,j \in [n]\,\label{def-ApproxScalar}
\end{align}
\vspace{-22pt}
\CompactEnumerate
\setcounter{enumi}{2}
\item  $g^3_{ij} := 0$; $\delta_{ij} = \norm{\av_i}\norm{\av_j}$,
\item  $g^4_{ij} \! \in_{\rm u.a.r.} \! [-\norm{\av_i}\norm{\av_j}, \norm{\av_i}\norm{\av_j}] $; $\delta_{ij} = \norm{\av_i}\norm{\av_j}$.
\end{list}

Oracle $g^1$ can be used in the rare cases where the dot product matrix is accessible to the optimization algorithm without any extra cost. In this case the updates will all be exact. If this matrix is not available, then the computation of each scalar product takes time $O(d)$. Hence, they cannot be recomputed on the fly, as argued in Section~\ref{sec-complexity}. In contrast, the oracles $g^3$ and $g^4$ are extremely cheap to compute, but the error bounds are worse. In the numerical experiments in Section~\ref{sec:emp_eval} we demonstrate that these oracles perform surprisingly well.

The oracle $g^2$ can for instance be realized by low-dimensional embeddings, such as given by the Johnson-Lindenstrauss lemma (cf. \citet{Achlioptas2003,Matousek:2008}). By embedding each vector in a lower-dimensional space of dimension $O\left( \epsilon^{-2} \log n\right)$ and computing the scalar products of the embedding in time $O(\log n)$, relation~\eqref{def-ApproxScalar} is satisfied.

\paragraph{Updating the gradient of the active coordinate.}
So far we only discussed the update of the passive coordinates. For the active coordinate the best strategy depends on the update rule $\mathcal{M}$ from~\eqref{def-M}. If exact line search is used, then $0 \in  \nabla_{i_t}f(\xv_{t+1})$. For other update rules we can update the gradient $\nabla_{i_t} f(\xv_{t+1})$ with the same gradient oracles as for the other coordinates, however we need also to take into account the change of the gradient of $\Psi_i([\xv_t]_i)$. If $\Psi_i$ is simple, like for instance in ridge or lasso, the subgradients at the new point can be computed efficiently.

\paragraph{Bounded variation.}
In many applications the Hessian $\nabla^2 f(A\tilde{\xv})$ is not so simple as in the case of square loss. If we assume that the Hessian of $f$ is bounded, i.e. $\nabla^2f(A\xv) \preceq M\cdot {I_n}$ for a constant $M \geq 0$, $\forall \xv \in \R^n$, then it is easy to see that the following holds : 
\begin{align*}
-M \| \av_i \| \| \av_j \| \leq \dotprod{\av_i, \nabla^2 f(A \tilde{\xv}) \av_{i_t}} \leq M \| \av_i \| \| \av_j \| \,.
\end{align*}
Using this relation, we can define gradient oracles for more general functions, by taking the additional approximation factor $M$ into account. The quality can be improved, if we have access to local bounds on $\nabla^2f(A\xv)$.

\paragraph{Heuristic variants.}
By design, ASCD is robust to high errors in the gradient estimations --  the steepest descent direction is always contained in the active set. 
However, instead of using only the very crude oracle $g^4$  to approximate \emph{all} scalar products, 
it might be advantageous to compute some scalar products with higher precision.
We propose to use a caching technique to compute the scalar products with high precision for all vectors in the active set (and storing a matrix of size $O(\mathcal{I}_t \times n)$).
This presumably works well if the active set does not change much over time.

\section{Extension to Composite Functions}
\label{sec:composite}
The key ingredients of ASCD 
are the coordinate-wise upper and lower bounds on the gradient and the definition of the active set $\mathcal{I}_t$ which ensures that the steepest descent direction is always kept and that only provably bad directions are removed from the active set. These ideas can also be generalized to the setting of composite functions~\eqref{def-composite}.
We already discussed some popular GS-$\ast$ update rules in the introduction in Section~\ref{sec-compositeintro}.

Implementing ASCD for the GS-s rule is straight forward, and we comment on the GS-r in the appendix in Sec.~\ref{sec:GS-r}. Here we exemplary detail the modification for the GS-q rule~\eqref{def-GSq}, which turns out to be the most evolved (the same reasoning also applies to the GSL-q rule from~\cite{nutini2015coordinate}).
In Algo.~\ref{alg-4} we show the construction --- based just on approximations of the gradient of the smooth part $f$ --- of the active set $\mathcal{I}$. 
For this we compute upper and lower bounds $\vv,\wv$ on $\min_{y \in \R} V(\xv,y,\nabla_i f(\xv))$, such that
\begin{align}
[\vv]_i \leq \min_{y \in \R} V(\xv,y,\nabla_i f(x) \leq [\wv]_i \quad \forall i \in [n]\,.
\end{align}
The selection of the active coordinate is then based on these bounds. 
Similar as in Lemma~\ref{lem-contained} and Theorem~\ref{thm-sandwich} this set has the property $i_{\rm GS-q} \in \mathcal{I}$, and 
directions are only discarded in such a way that the efficiency of ASCD-q  cannot drop below the efficiency of UCD. 
The proof 
can be found in the appendix in Section~\ref{sec:GSq}.

\begin{algorithm}[tb]
   \caption{Adaptation of ASCD for GS-q rule}
   \label{alg-4}
   \renewcommand{\algorithmiccomment}[1]{\hfill {\footnotesize\textit{#1}} \null}   
\begin{algorithmic}
   \STATE {\bfseries Input:} Gradient estimate $\gtv$, error bounds $\rv$.
   \STATE For $i\in[n]$ define: \COMMENT{compute u.-and l.-bounds}
   \STATE $[\uv]_i := [\gtv]_i + [\rv]_i$, $[\ellv]_i := [\gtv]_i - [\rv]_i$ \\[1ex]
   \STATE $[\usv]_i := \argmin_{y \in \R} V(\xv,y, [\uv]_i)$ \COMMENT{minimize the model}
   \STATE $[\ellsv]_i := \argmin_{y \in \R} V(\xv,y, [\ellv]_i)$ \\[1ex]
   \STATE \COMMENT{compute u.-and l. bounds on $\min_{y \in \R} V(\xv,y,\nabla_i f(\xv))$}
   \STATE $[\omegav_u]_i := V(\xv, [\usv]_i, [\uv]_i)  \! + \! \max\{0, [\usv]_i([\ellv]_i - [\uv]_i)\} $
   \STATE $[\omegav_\ell]_i  := V(\xv, [\ellsv]_i, [\ellv]_i)  + \max\{0, [\ellsv]_i([\uv]_i - [\ellv]_i)\} $    
   \STATE $[\vv]_i := \min \left\{ V(\xv, [\usv]_i, [\uv]_i),  V(\xv, [\ellsv]_i, [\ellv]_i) \right\}$
   \STATE $[\wv]_i := \min\left\{[\omegav_u]_i, [\omegav_\ell]_i, \Psi_i([\xv]_i) \right\}$ \\[1ex]
    \STATE $\operatorname{av}({\mathcal{I}}) := \frac{1}{\abs{\mathcal{I}}} \sum_{i \in \mathcal{I}} [\wv]_i$ \COMMENT{compute active set}
   \STATE $\mathcal{I}_t := \argmin_{\mathcal{I}} \abs{\left\{\mathcal{I} \subseteq [n] \mid [\vv]_i > \operatorname{av}({\mathcal{I}}), \forall i \notin \mathcal{I}\right\}}$
\end{algorithmic}
\end{algorithm}

\section{Analysis of Competitive Ratio}%
\label{sec:competitiveRatio}
In Section~\ref{sec:algorithm} we derived in Thm.~\ref{thm-sandwich} that the one step progress of ASCD is between the bounds on the onestep progress of UCD and SCD. However, we know that the efficiency of the latter two methods can differ much, up to a factor of $n$. In this section we will argue that in certain cases where SCD performs much better than UCD, ASCD will accelerate as well. To measure this effect, we could for instance consider the ratio:
\begin{align}
\varrho_t := \frac{\abs{ \left\{i \in \mathcal{I}_t \mid \abs{\nabla_i f(\xv_t)} \geq \frac{1}{2} \norm{\nabla f(\xv_t)}_\infty \right\}}}{\abs{\mathcal{I}_t}}\,, \label{eq-compR2}
\end{align}
For general functions this expression is a bit cumbersome to study, therefore we restrict our discussion to the class of objective functions~\eqref{def-lowdim} as introduced in Sec.~\ref{sec-lowerbound}. 
Of course not all real-world objective functions will fall into this class, however this problem class is still very interesting in our study, as we will see in the following, because it will highlight the ability (or disability) of the algorithms to eventually identify the right set of  `active' coordinates.

For the functions with the structure~\eqref{def-lowdim} (and $q$ as in Thm.~\ref{thm-lowerbound}), the active set falls into the first $s$ coordinates. 
Hence it is reasonable to approximate $\varrho_t$ by the competitive ratio
\vspace{-5pt} 
\begin{align}
\rho_t := \frac{\abs{\mathcal{I}_t \cap [s]}}{\abs{\mathcal{I}_t}}\,. \label{eq-compR}
\end{align}
It is also reasonable to assume that in the limit, $(t \to \infty)$, a constant fraction of the $[s]$ will be contained in the active set~$\mathcal{I}_t$ (it might not hold $[s] \subseteq \mathcal{I}_t$ $\forall t$, as for instance with exact line search the directional derivative vanishes just after the update). In the following theorem we calculate %
$\rho_t$ for $(t \to \infty)$, the proof is given in the appendix. 

\begin{theorem}
\label{thm-competitiveRatio2}
Let $f \colon \R^n \to \R$ be of the form~\eqref{def-lowdim}. For indices $i \notin [s]$ define $\mathcal{K}_i := \left\{t \mid i \notin \mathcal{I}_t, i \in \mathcal{I}_{t-1} \right\}$. For $j \in \mathcal{K}_i$  define $T_j^i := \min \left\{t - j\mid i \in \mathcal{I}_{j+t} \right\}$, i.e. the number of iterations outside the active set, $T^i_\infty := \lim_{t \to \infty}\EE{j \in \mathcal{K}_i}{T_j^i \mid j > k}$, and the average $T_\infty:=  \EE{i \notin [s]}{T_\infty^i}$. %
If there exists a constant $c > 0$ such that $\lim_{t \to \infty} \abs{[s] \cap \mathcal{I}_t} = cs$, then (with the notation $\rho_\infty := \lim_{t \to \infty} \E{\rho_t}$),
\vspace{-3pt} 
\begin{align}
\rho_\infty \geq \frac{2 c s}{cs + n -s - T_\infty + \sqrt{\theta}} \,, \label{eq-comptRexact}
\end{align}
where $\theta \equiv \theta := n^2 + (c-1)^2s^2 + 2n((c-1)s - T_\infty) + 2(1+c)sT_\infty +T_\infty^2$.
Especially,  $\rho_\infty \geq 1 - \frac{n-s}{T_\infty}$.
\end{theorem}
In Figure~\ref{fig-competitiveR} we compare the lower bound~\eqref{eq-comptRexact} of the competitive ratio in the limit ($t\to \infty$) with actual measurements of $\rho_t$ for simulated example with parameters 
$n=100$, $s=10$, $c=1$ and various $T_\infty \in \{50,100,400\}$. We initialized the active set $\mathcal{I}_0 = [s]$, but we see that the equilibrium is reached quickly.

\begin{figure}[tb]
  \centering
\includegraphics[width=0.95\columnwidth]{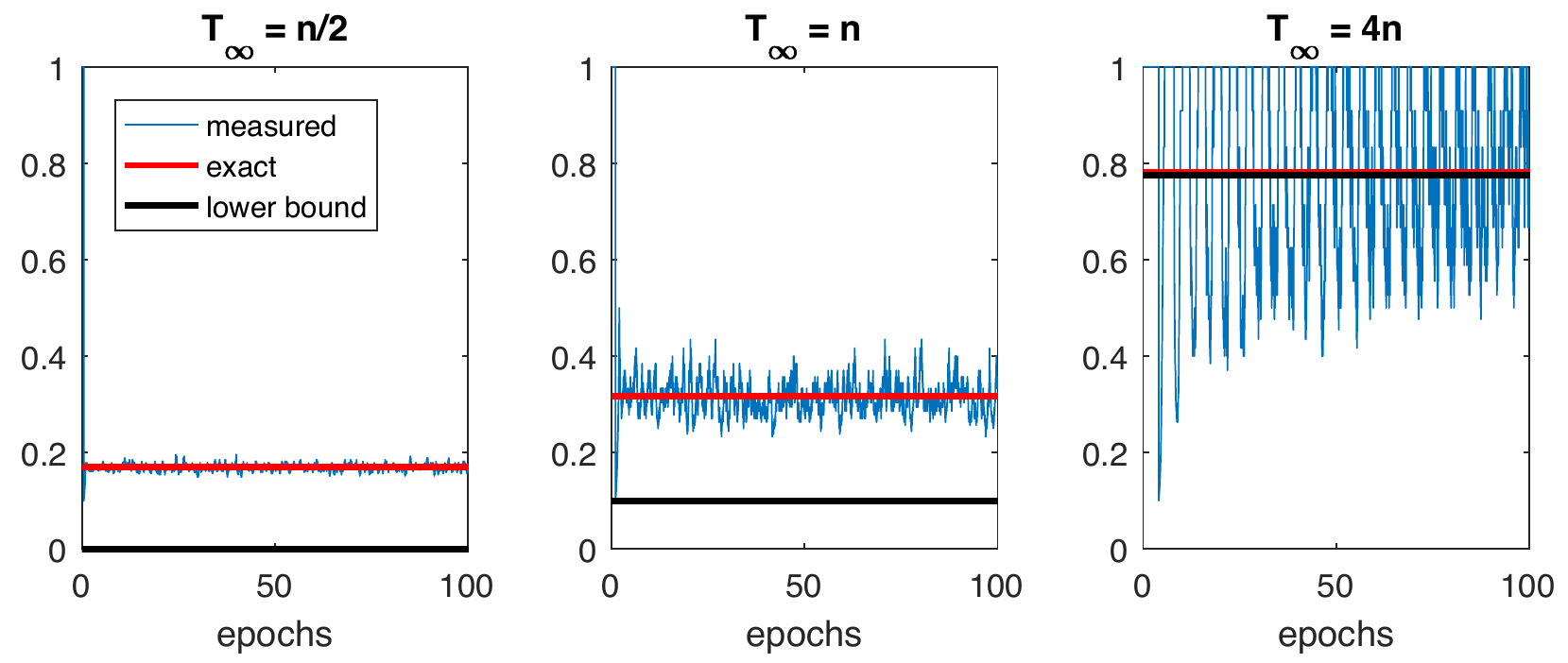}
\vskip-10pt
\caption{Competitive ratio $\rho_t$ (\emph{blue}) in comparison with $\rho_\infty$~\eqref{eq-comptRexact} (\emph{red}) and the lower bound  $\rho_\infty \geq 1 - \frac{n-s}{T_\infty}$ (\emph{black}). Simulation for parameters $n=100$, $s=10$, $c=1$ and $T_\infty \in \{50,100,400\}$.  }
   \vspace{-1mm}
\label{fig-competitiveR}
\end{figure}

\begin{figure*}[ht]  
\centering
\begin{subfigure}{.25\textwidth}
  \centering
  \includegraphics[width=4.7cm,height=3.8cm]{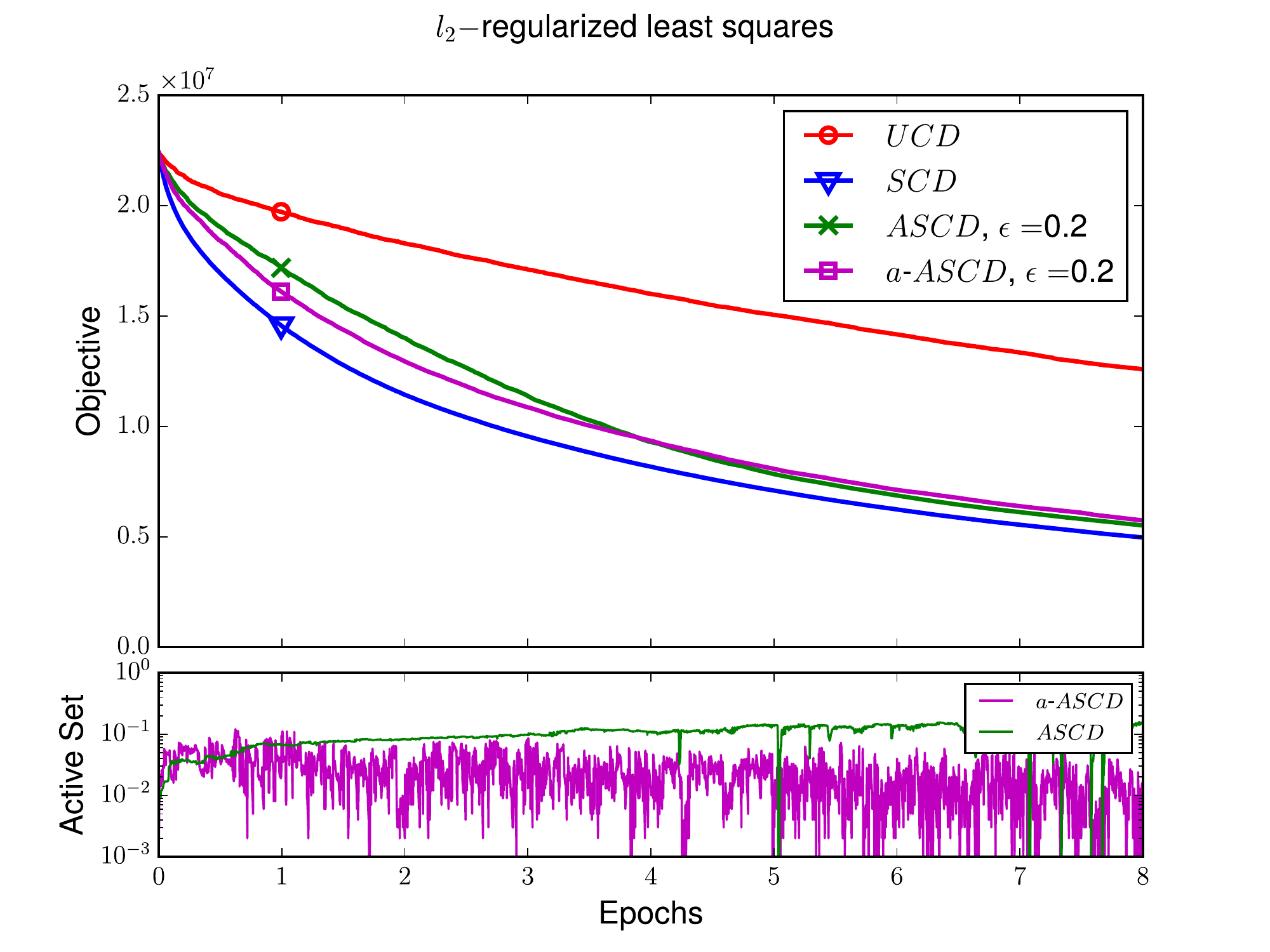}
   \vspace{-7mm}
  \caption{Convergence for $l_2$ }
  \label{fig:simul_l2}
\end{subfigure}%
\begin{subfigure}{.25\textwidth}
  \centering
  \includegraphics[width=4.7cm,height=3.8cm]{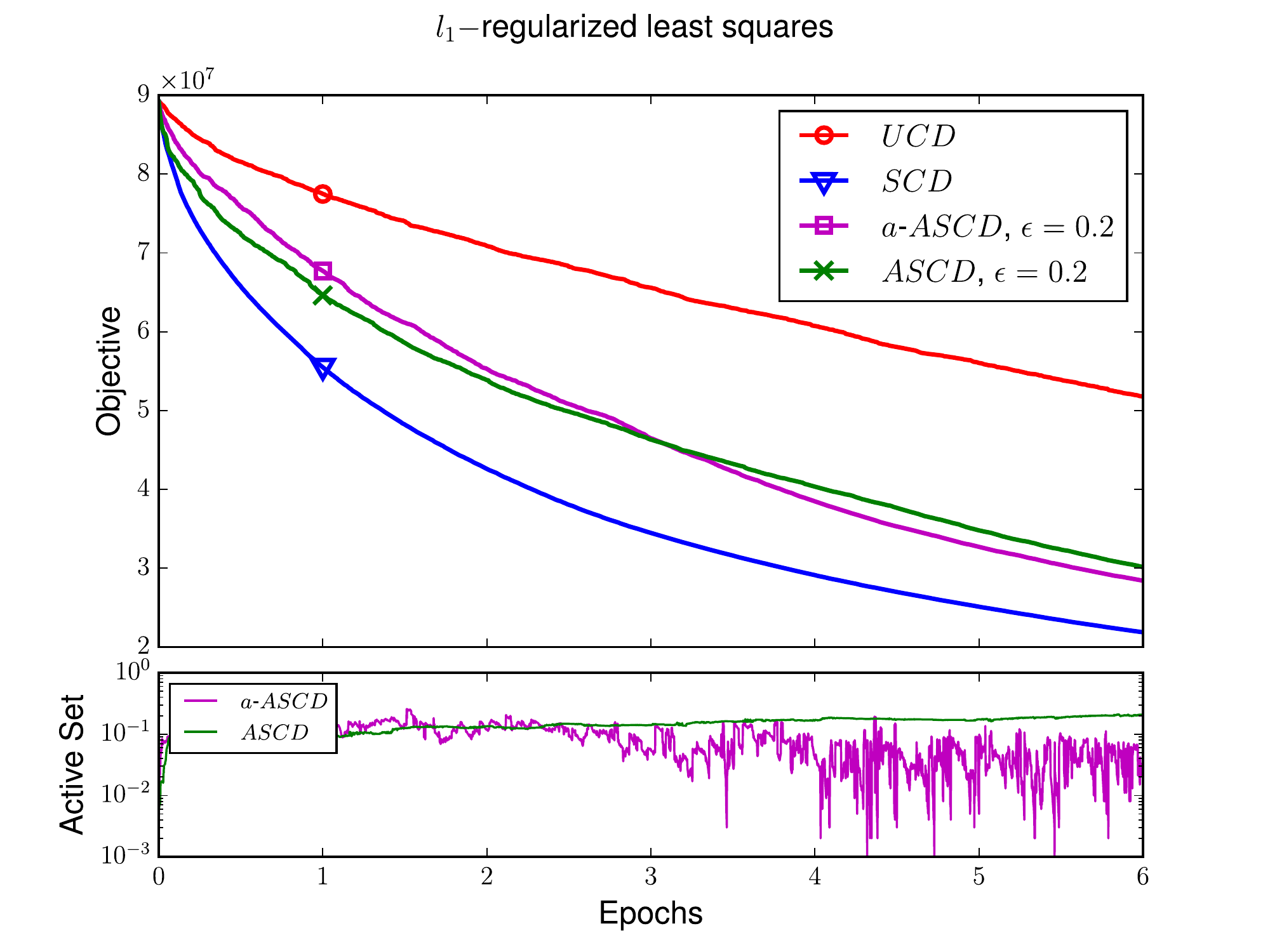}
  \vspace{-7mm}
  \caption{Convergence for $l_1$ }
  \label{fig:simul_l1}
\end{subfigure}%
\begin{subfigure}{.25\textwidth}
  \centering
  \includegraphics[width=4.7cm,height=3.8cm]{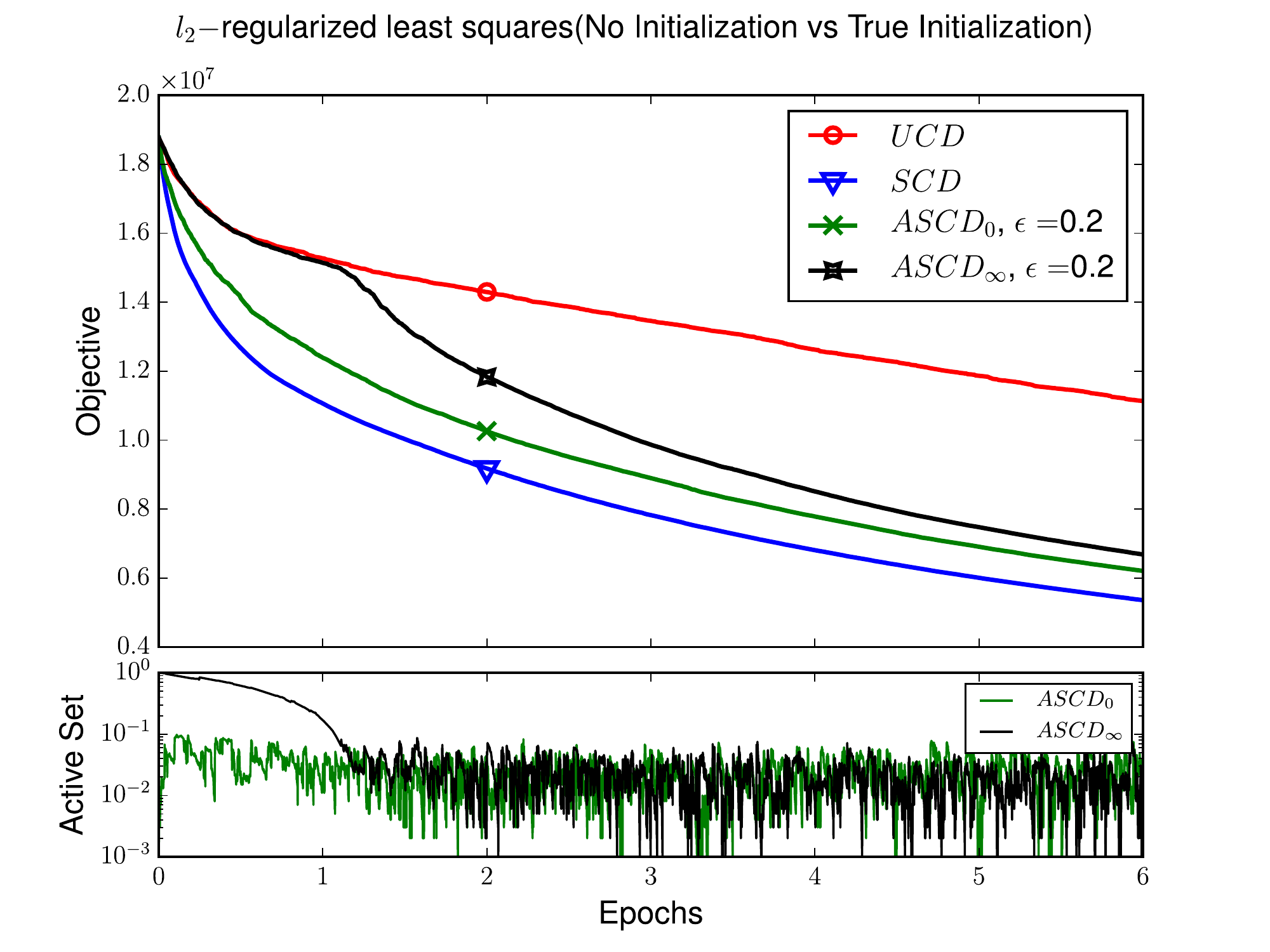}
  \vspace{-7mm}
  \caption{True vs No Initialization for $l_2$ }
  \label{fig:simul_t_vs_n}
\end{subfigure}%
\begin{subfigure}{.25\textwidth}
  \centering
  \includegraphics[width=4.7cm,height=3.8cm]{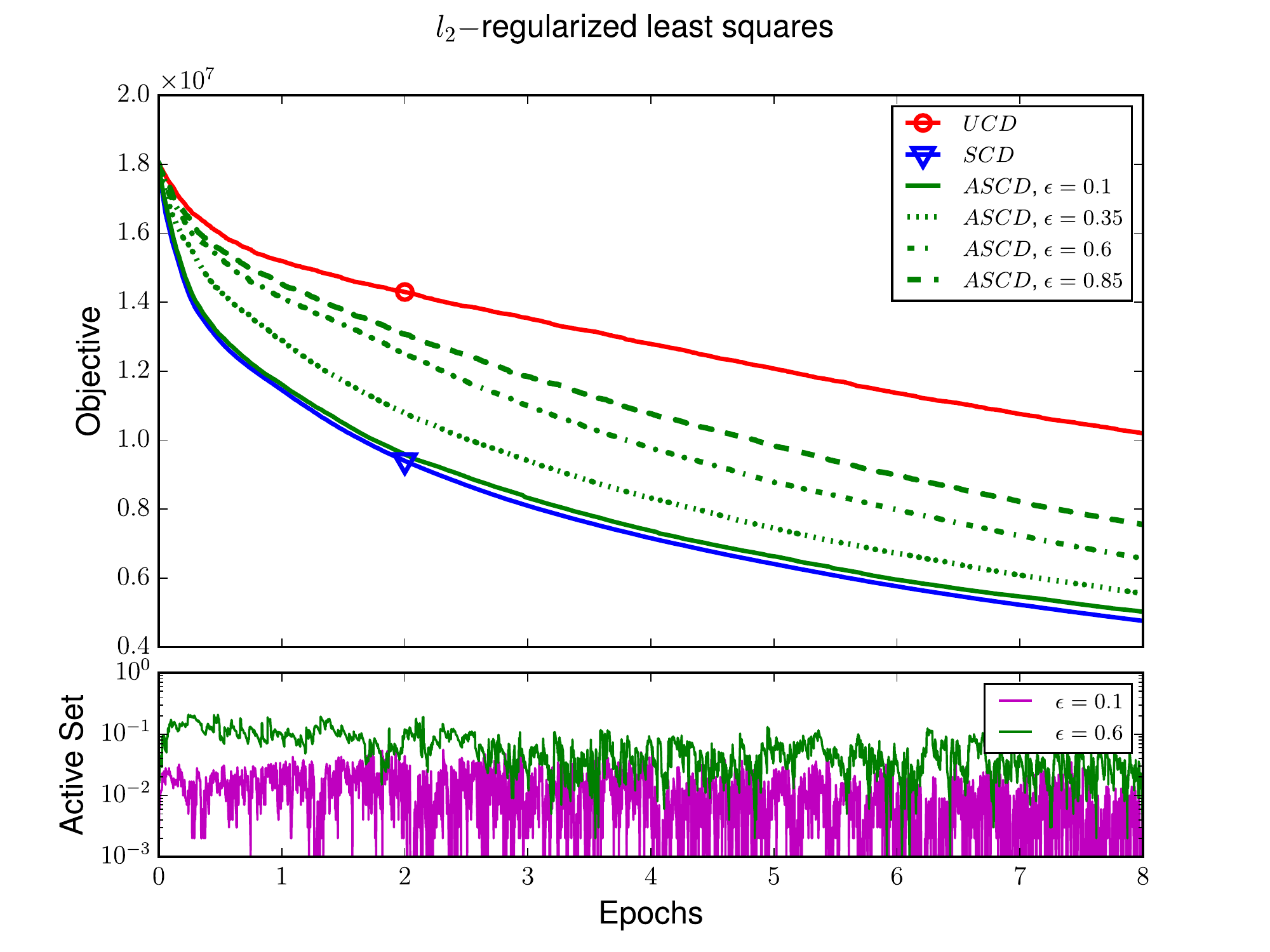}
  \vspace{-7mm}
  \caption{Error Variation (ASCD)}
  \label{fig:simul_error}
\end{subfigure}%
\vspace{-1mm}
\caption{Experimental results on synthetically generated datasets}
\label{fig:results_synthetic}
\end{figure*}

\begin{figure*}[ht] 
\centering
\begin{subfigure}{.25\textwidth}
  \centering
  \includegraphics[width=4.7cm,height=3.8cm]{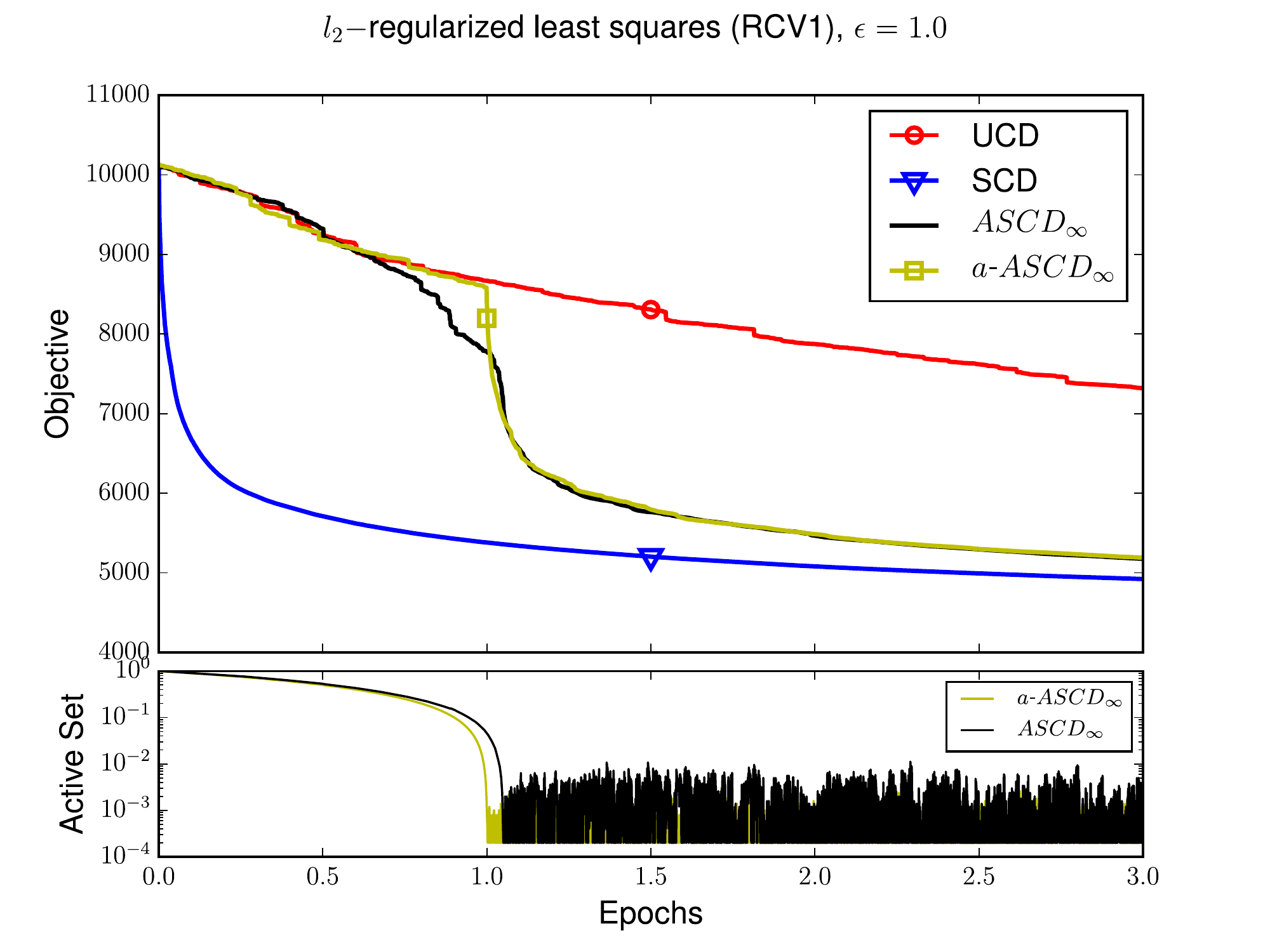}
  \vspace{-7mm}
  \caption{Convergence for $l_2$ }
  \label{fig:rcv_l2}
\end{subfigure}%
\begin{subfigure}{.25\textwidth}
  \centering
  \includegraphics[width=4.7cm,height=3.8cm]{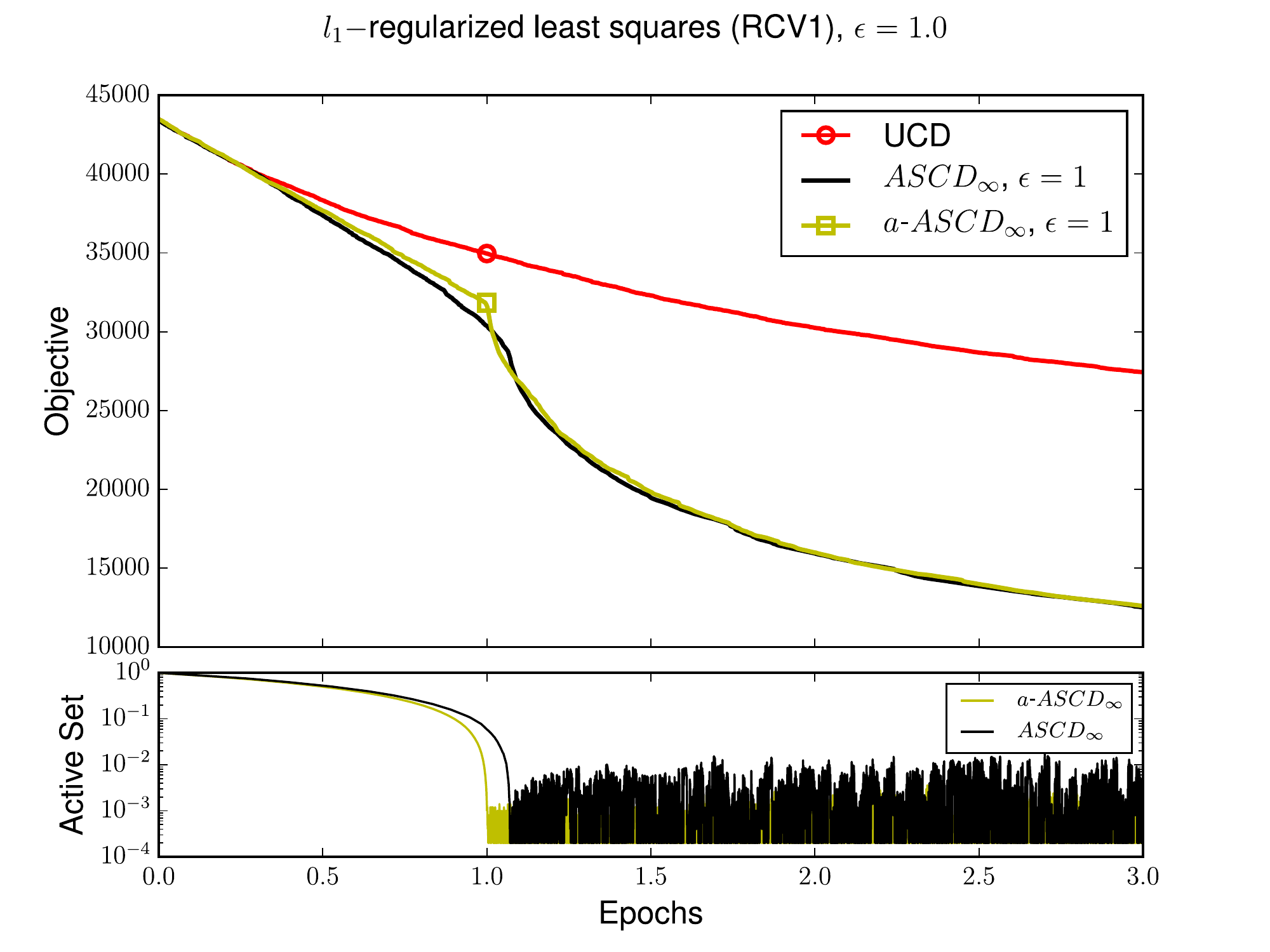}
  \vspace{-7mm}
  \caption{Convergence for $l_1$ }
  \label{fig:rcv_l1}
\end{subfigure}%
\begin{subfigure}{.25\textwidth}
  \centering
  \includegraphics[width=4.7cm,height=3.8cm]{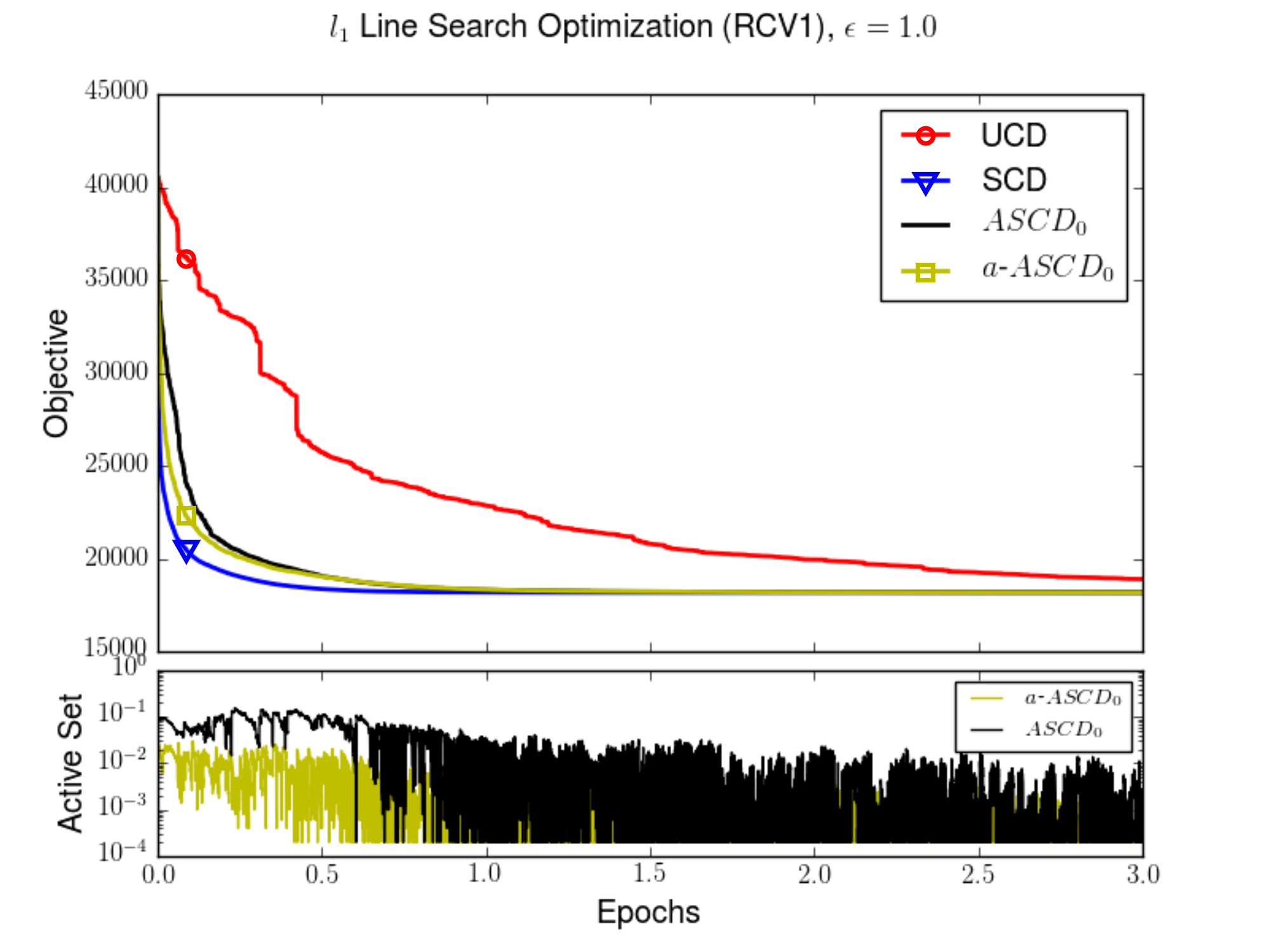}
  \vspace{-7mm}
  \caption{Line search for $l_1$}
  \label{fig:rcv_figc}
\end{subfigure}%
\begin{subfigure}{.25\textwidth}
  \centering
  \includegraphics[width=4.7cm,height=3.8cm]{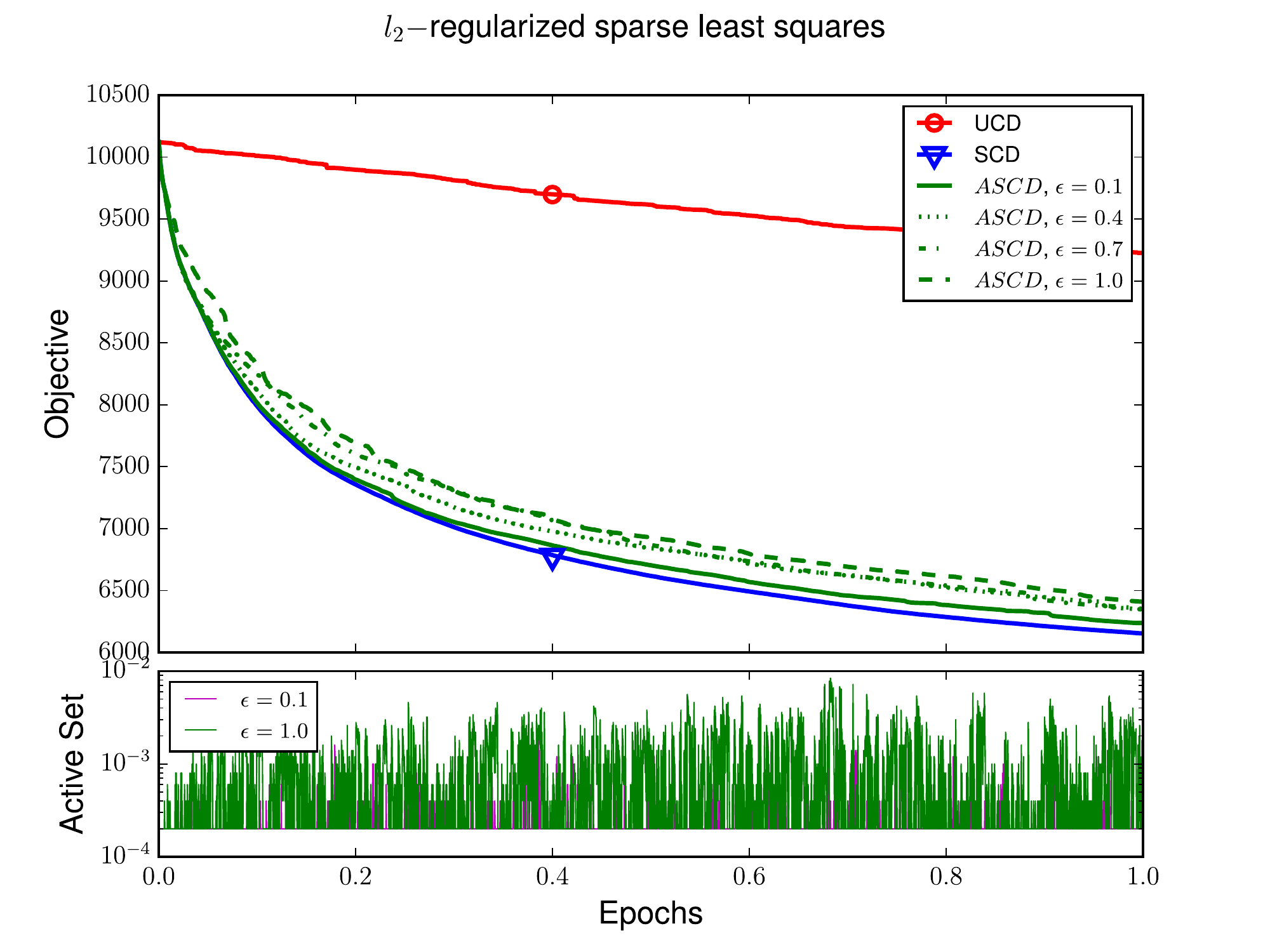}
  \vspace{-7mm}
  \caption{Error Variation (ASCD)}
   \label{fig:rcv_error}
\end{subfigure}%
\vspace{-1mm}
\caption{Experimental results on the RCV1-binary dataset}
 \label{fig:results_rcv1}
\end{figure*}

\subsection{Estimates of the competitive ratio}
Based on this Thm.~\ref{thm-competitiveRatio2}, we can now estimate the competitive ratio in various scenarios. On the class~\eqref{def-lowdim} it holds $c \approx 1$ as we argued before. Hence the competitive ratio~\eqref{eq-comptRexact} just depends on $T_\infty$. This quantity measures how many iterations a coordinate $j \notin [s]$ is in average outside of the active set $\mathcal{I}_t$. From the lower bound 
we see that the competitive ratio $\rho_t$ approaches a constant for ($t \to \infty$) if $T_\infty = \Theta\left( n \right)$, for instance $\rho_\infty \geq 0.8$ if $T_\infty \geq 5n$.

As an approximation to $T_\infty$, we estimate the quantities $T_{t_0}^j$ defined in Thm.~\ref{thm-competitiveRatio2}. $T_{t_0}^j$ denotes the number of iterations it takes until coordinate $j$ enters the active set again, assuming it left the active set at iteration $t_0-1$.  We estimate $T_{t_0}^j \geq \hat{T}$, where $\hat{T}$ denotes maximum number of iterations such that
\vspace{-1mm}
\begin{align}
 \sum_{t=t_0}^{t_0 + \hat{T}} \gamma_{t} \delta_{i_i j}\leq  \frac{1}{s} \sum_{k =1}^s \abs{\nabla_k f \left(\xv_{t_0 + \hat{T}}\right)} \quad  \forall j \notin [s]. \label{eq-detail}
\end{align}

For smooth functions, the steps $\gamma_t = \Theta \left( \abs{\nabla_{i_t} f(\xv_{t})} \right)$ and if we additionally assume that the errors of the gradient oracle are uniformly bounded $\delta_{ij} \leq \delta$, the sum in~\eqref{eq-detail} simplifies to $\delta  \sum_{t=t_0}^{t_0 + \hat{T}}  \abs{\nabla_{i_t} f(\xv_t)}$.

For smooth, but not strongly convex function $q$, the norms of the gradient changes very slowly, with a rate independent of $s$ or $n$, and we get $\hat{T} =  \Theta\left( \frac{1}{\delta} \right)$. Hence, the competitive ratio is constant for  $\delta = \Theta\left( \frac{1}{n} \right)$.

For strongly convex function $q$, the norm of the gradient decreases linearly, say $\norm{\nabla f(\xv_t)}_2^2 \propto e^{\kappa t}$ for $\kappa \approx \frac{1}{s}$. I.e. it decreases by half after each $  \Theta\left( s \right)$ iterations. Therefore to guarantee $\hat{T} = \Theta\left( n \right)$ it needs to hold $\delta = e^{-\Theta\left( \frac{n}{s} \right)}$. This result seems to indicate that the use of ACDM is only justified if $s$ is large, for instance $s \geq \frac{1}{4}n$. Otherwise the convergence on $q$ is too fast, and the gradient approximations are too weak. 
However, notice that we assumed  $\delta$ to be an uniform bound on all errors. If the errors have large discrepancy the estimates become much better (this holds for instance on datasets where the norm data vectors differs much, or when caching techniques as mentioned in Sec.~\ref{sec:gradUpdates} are employed).  
\section{Empirical Observations} \label{sec:emp_eval}
In this section we evaluate the empirical performance of ASCD on synthetic and  real datasets. 
We consider the following regularized general linear models:
\begin{align}
\min_{\xv \in \R^n} \tfrac{1}{2} \| A\xv - \bv\|_2^2 + \tfrac{\lambda}{2} \| \xv \|_2^2 \,,  \label{eq:emp_problem1} \\
\min_{\xv \in \R^n} \tfrac{1}{2} \| A\xv - \bv \|_2^2 + \lambda \| \xv \|_1 \,, \label{eq:emp_problem2}
\end{align}
that is, $l_2$-regularized least squares~\eqref{eq:emp_problem1} as well as $l_1$-regularized linear regression (Lasso) in~\eqref{eq:emp_problem2}, respectively.

\paragraph{Datasets.} The datasets $A \in \mathbb{R}^{d \times n}$ in problems~\eqref{eq:emp_problem1} and~\eqref{eq:emp_problem2} were chosen as follows for our experiments. For the synthetic data, we follow the same  generation procedure as described in~\cite{nutini2015coordinate}, which generates very sparse data matrices. 
For completeness, full details of the data generation process are also provided in the appendix in Sec.~\ref{sec:datageneration}. For the synthetic data 
we choose $n=5000$ for problem~\eqref{eq:emp_problem2} and $n = 1000$ for problem~\eqref{eq:emp_problem1}. Dimension $d=1000$ is fixed for both cases.\\
For real datasets, we perform the experimental evaluation on RCV1 (binary,training), which consists of $20,242$ samples, each of dimension $47,236$~\cite{Lewis:2004:RNB}. We use the un-normalized version with all non-zeros values set to $1$ (bag-of-words features).

\paragraph{Gradient oracles and implementation details.}
On the RCV1 dataset, we approximate the scalar products with the oracle $g^4$ that was introduced in Sec.~\ref{sec:gradUpdates}. This oracle is extremely cheap to compute, as the norms $\norm{\av_i}$ of the columns of $A$ only need to be computed once.\\
On the synthetic data, we simulate the oracle $g^2$ for various precisions values $\epsilon$. For this, we sample a value uniformly at random from the allowed error interval~\eqref{def-ApproxScalar}. Figs.~\ref{fig:simul_error} and~\ref{fig:rcv_error} show the convergence for different accuracies. \\
For the $l_1$-regularized problems, we used ASCD with the GS-s rule (the experiments in~\cite{nutini2015coordinate} revealed almost identical performance of the different GS-$\ast$ rules).\\
We compare the performance of UCD, SCD and ASCD. We also implement the heuristic version a-ASCD that was introduced in Sec.~\ref{sec:algorithm}. All algorithm variants use the same step size rule (i.e. the method $\mathcal{M}$ in Algorithm~\ref{alg-1}). We use exact line search for the experiment in Fig.~\ref{fig:rcv_figc}, for all others we used a fixed step size rule (the convergence is slower for all algorithms, but the different effects of the selection of the active coordinate is more distinctly visible).\\
ASCD is either initialized with the true gradient (Figs.~\ref{fig:simul_l2}, \ref{fig:simul_l1}, \ref{fig:simul_error}, \ref{fig:rcv_figc}, \ref{fig:rcv_error}) or  arbitrarely (with error bounds $\delta= \infty$) in Figs.~\ref{fig:rcv_l2} and~\ref{fig:rcv_l1} (Fig.~\ref{fig:simul_t_vs_n} compares both initializations). \\
Fig.~\ref{fig:results_synthetic} shows results on the synthetic data, Fig.~\ref{fig:results_rcv1} on the RCV1 dataset. All plots show also the size of the active set $\mathcal{I}_t$. The plots~\ref{fig:rcv_figc} and \ref{fig:rcv_error} are generated on a subspace of RCV1, with $10000$ and $5000$ randomly chosen columns, respectively.

Here are the highlights of our experimental study:
\CompactEnumerate
\item  \textbf{No initialization needed.} We observe (see e.g. Figs.~\ref{fig:simul_t_vs_n},\ref{fig:rcv_l2}, \ref{fig:rcv_l1}) that initialization with the true gradient values is \emph{not} needed at beginning of the optimization process (the cost of the initialization being as expensive as one epoch of ASCD). 
Instead, the algorithm performs strong in terms of learning the active set on its own, and the set converges very fast after just one epoch.
\item  \textbf{High errors toleration.}
 The gradient oracle $g^4$ gives very crude approximations, however the convergence of ASCD is excellent on RCV1 (Fig.~\ref{fig:results_rcv1}).  Here the size of the true active set is very small (in the order of $0.1\%$ on RCV1) and ASCD is able to identify this set. Fig.~\ref{fig:rcv_error} shows that almost nothing can be gained from more precise (and more expensive) oracles.
\item  \textbf{Heuristic a-ASCD performs well.}
The convergence behavior of ASCD follows theory. For the heuristic version a-ASCD (which computes the active set slightly faster, but Thm.~\ref{thm-sandwich} does not hold) performs identical to ASCD in practice (cf. Figs.~\ref{fig:results_synthetic}, \ref{fig:results_rcv1}), and sometimes slightly better. This is explained by the active set used in ASCD typically being larger than the active set of a-ASCD (Figs.~\ref{fig:simul_l2},\ref{fig:simul_l1}, \ref{fig:rcv_l2}, \ref{fig:rcv_l1}).
\end{list}

\section{Concluding Remarks}
\label{sec:conclusion}
\vspace{-1mm}

We proposed ASCD, a novel selection mechanism for the active coordinate in CD methods. Our scheme enjoys three favorable properties: 
(i)~its performance can reach the performance steepest CD --- both in theory and practice, 
(ii)~the performance is never worse than uniform CD,
(iii)~in many important applications, the scheme it can be implemented at no extra cost per iteration.

ASCD calculates the active set in a safe manner, and picks the active coordinate uniformly at random from this smaller set. It seems possible that an adaptive sampling strategy on the active set could boost the performance even further.
Here we only study CD methods where a single coordinate gets updated in each iteration. ASCD can immediately also be generalized to block-coordinate descent methods. However, the exact implementation in a distributed setting can be challenging.

Finally, it is an interesting direction to extend ASCD also to the stochastic gradient descent setting (not only heuristically, but with the same strong guarantees as derived in this paper).

{\small
\bibliographystyle{icml2017}
\bibliography{opt-ml}
}


\newpage
\appendix

\clearpage
\onecolumn
\begin{center}
{\centering \LARGE Appendix }
\vspace{1cm}
\end{center}
\section{On Steepest Coordinate Descent} \label{appen:steep}
\subsection{Convergence on Smooth Functions}
\label{sec-proofs-for-smooth}

\begin{lemma}[Lower bound on the one step progress on smooth functions]\label{lem:smooth_steepest}
Let $f \colon \R^n \to \R$ be convex and coordinate-wise $L$-smooth. For a sequence of iterates $\{\xv_t\}_{t \geq 0}$ define the progress measure
\begin{align}
\Delta(\xv_t) := \frac{1}{ \E{ f(\xv_{t+1}) - f(\xv^\star) \mid \xv_t}} -  \frac{1}{ f(\xv_{t}) - f(\xv^\star)}\,.
\end{align}
For sequences $\{\xv_t\}_{t \geq 0}$ generated by SCD it holds:
\begin{align}
 \Delta_{\rm SCD}(\xv_t) &\geq  \frac{1}{2L \norm{ \xv_t - \xv^\star }_1^2} \,, & &t \geq 0\,,
 \label{eq:steep_converg_smooth}
 \end{align}
and for a sequences generated by UCD:
\begin{align}  
 \Delta_{\rm UCD}(\xv_t) &\geq  \frac{1}{2nL \norm{ \xv_t - \xv^\star  }_2^2} \,, & &t \geq 0\,.
 \label{eq:uni_converg_smooth}
\end{align}
\end{lemma}
It is important to note that the lower bounds presented in Equations~\eqref{eq:steep_converg_smooth} and \eqref{eq:uni_converg_smooth} are quite tight and equality is almost achievable under special conditions. When comparing the per-step progress of these two methods, we find --- similarly as in~\eqref{eq-compareSU} --- the relation
\begin{align}
 \frac{1}{n} \Delta_{\rm SCD}(\xv_t) \leq \Delta_{\rm UCD}(\xv_t) \leq \Delta_{\rm SCD}(\xv_t)\,, \label{eq-compareDU}
\end{align}
that is, SCD can boost the performance over the random coordinate descent up to the factor of $n$. This also holds for a sequence of consecutive updates, as show in Theorem~\ref{thm:smooth_steep_converge}.

\begin{proof}[\textbf{Proof of Lemma}~\ref{lem:smooth_steepest}]
Define $f^\star := f(\xv^\star)$.
From the smoothness assumption~\eqref{eq-Ubound}, we get
\begin{align}
&f(\xv_{t+1}) \refLE{eq-quadONESTEP}  f(\xv_{t})  - \frac{1}{2L} \|\nabla f(\xv_t) \|_{\infty}^2 \notag \\
\Rightarrow~ &\big( f(\xv_{t+1}) - f^\star\big) \leq \big( f(\xv_{t}) - f^\star \big)  - \frac{1}{2L} \|\nabla f(\xv_t) \|_{\infty}^2 \label{eq:conv-sgcd-p1}
\end{align}

Now from the property of a convex function and H\"{o}lder's inequality:
\begin{align}
f(\xv_{t}) - f^\star& \leq \langle \nabla f(\xv_t), \xv_t - \xv^\star  \rangle  \leq \| \nabla f(\xv_t) \|_{\infty} \| \xv_t - \xv^\star  \|_1
\end{align}
Hence, 
\begin{align}
&\big( f(\xv_{t}) - f^\star \big)^2 \leq \| \nabla f(\xv_t) \|_{\infty}^2 \| \xv_t - \xv^\star  \|_1^2 \notag \\
\Rightarrow~ &\|\nabla f(\xv_t) \|_{\infty}^2 \geq  \frac{\big( f(\xv_{t}) - f^\star \big)^2}{\| \xv_t - \xv^\star  \|_1^2}   \label{eq:conv-sgcd-p2}
\end{align}

From Equations \eqref{eq:conv-sgcd-p1} and \eqref{eq:conv-sgcd-p2},
\begin{align}
 \frac{1}{\big( f(\xv_{t+1}) - f^\star \big)} -  \frac{1}{\big( f(\xv_{t}) - f^\star \big)} \geq \frac{1}{2L \| \xv_t - \xv^\star  \|_1^2}  \label{eq:conv-steep-coord}
\end{align}
Which concludes the proof.
\end{proof}
We like to remark, that the one step progress for UCD can be written as~\cite{nesterov2012,wright2015}:
\begin{align} \label{eq:uni_converg_smooth-again}
\frac{1}{\big( \mathbb{E} [f(\xv_{t+1}) | \xv_t] - f^\star \big)} -  \frac{1}{\big(f(\xv_{t}) - f^\star \big)} \geq \frac{1}{2Ln \| \xv_t - \xv^\star  \|_2^2} 
\end{align}

\begin{proof}[\textbf{Proof of Theorem}~\ref{thm:smooth_steep_converge}]
From Lemma~\ref{lem:smooth_steepest}, $$ \frac{1}{\big( f(\xv_{t+1}) - f^\star \big)} -  \frac{1}{\big( f(\xv_{t}) - f^\star \big)} \geq \frac{1}{2L \| \xv_t - \xv^\star  \|_1^2} $$
Now summing up the above equation for $t =0$ till $t-1$, we get:
\begin{align*}
& \frac{1}{\big( f(\xv_{t}) - f^\star \big)}  - \frac{1}{\big( f(\xv_{0}) - f^\star \big)}  \geq \frac{1}{2L} \sum_{i = 0}^{t-1} \frac{1}{\| \xv_t - \xv^\star \|_1^2} \\
 \Rightarrow~ & \frac{1}{\big( f(\xv_{t}) - f^\star \big)}   \geq \frac{1}{2L} \sum_{i = 0}^{t-1} \frac{1}{\| \xv_0 - \xv^\star \|_1^2} \\
 \Rightarrow~ & \frac{1}{\big( f(\xv_{t}) - f^\star \big)}   \geq \frac{t}{2L R_1^2}  \\
  \Rightarrow~ & f(\xv_t) - f^\star \leq \frac{2 L R_1^2}{t}  
\end{align*}
Which concludes the proof.
\end{proof}

\subsection{Lower bounds}
\label{sec-lowerbound-proof}
In this section we provide the proof of Theorem~\ref{thm-lowerbound}. 
Our result is slightly more general, we will proof the following (and Theorem~\ref{thm-lowerbound} follows by the choice $\alpha = 0.01 < \frac{1}{3}$).

\begin{theorem}
\label{thm-lowerbound2}
Consider the function $q(\xv) = \frac{1}{2} \dotprod{Q\xv,\xv}$ for $Q := (\alpha -1)\frac{1}{n} J_n + I_n$, where $J_n = \bm{1}_n \bm{1}_n^T$ and $0 < \alpha < \frac{1}{2}$, $n > 2$. Then there exists $\xv_0 \in R^n$ such that for the sequence $\{\xv_t\}_{t \geq 0}$ generated by SCD it holds
\begin{align}
 \norm{ \nabla q(\xv_t) }_\infty^2 \leq 
 \frac{3 + 3 \alpha }{n} 
  \norm{ \nabla q(\xv_t) }_2^2\,. \label{eq-lowerboundratio2}
\end{align}
\end{theorem}
In the proof below we will construct a special $\xv_0 \in R^n$ that has the claimed property. However, we would like to remark that this is not very crucial. We observe that for functions as in Theorem~\ref{thm-lowerbound2} almost any initial iterate ($\xv$ not aligned with the coordinate axes) the sequence $\{\xv_t\}_{t \geq 0}$ of iterates generated by SCD suffers from the same issue, i.e. relation~\eqref{eq-lowerboundratio2} holds for iteration counter $t$ sufficiently large. We do not prove this formally, but demonstrate this behavior in Figure~\ref{fig:lb}. We see that the steady state is almost reached after $2n$ iterations.

\begin{figure}
\begin{center}
\includegraphics[scale=.5]{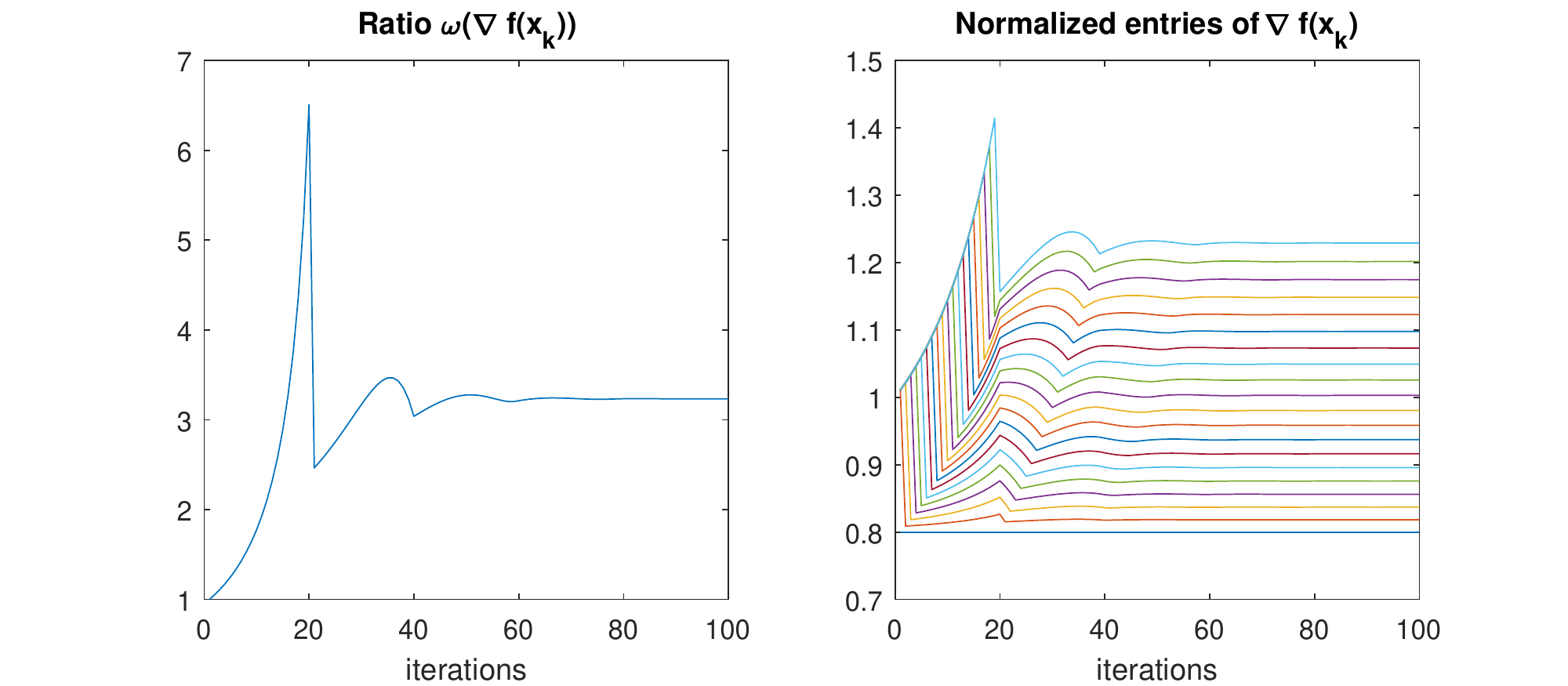}
\end{center}
\caption{SCD on the function from Theorem~\ref{thm-lowerbound2} in dimension $n= 20$ with $\xv_0 = \bm{1}_n$ (i.e. not the worst starting point constructed in the proof of Theorem~\ref{thm-lowerbound2}). On the right the (normalized and sorted) components of $\nabla f(\xv_t)$.}
\label{fig:lb}
\end{figure}

\begin{proof}[Proof of Theorem~\ref{thm-lowerbound2}]
Define the parameter $c_\alpha$ by the equation
\begin{align}
 \left( 1 + \frac{ \alpha-1}{n} \right) c_\alpha^{n-1} = \left( \frac{1- \alpha}{n} \right) S_{n-1}(c_\alpha) \label{eq-defC} \\
 c_\alpha^{n-1} = \left( \frac{1- \alpha}{n} \right) S_{n}(c_\alpha)
\end{align}
where $S_n(c_\alpha) = \sum_{i=0}^{n-1} c_\alpha^n$; and define $\xv_0$ as $[\xv_0]_i = c_\alpha^{i-1}$ for $i=1,\dots,n$.  In Lemma~\ref{lemma-lowbound1} below we show that $c_\alpha \geq 1 - \frac{3}{n}\alpha$. 

We now show that SCD cycles through the coordinates, i.e. the sequence $\{\xv_t\}_{t \geq 0}$ generated by SCD satisfies 
\begin{align}
[\xv_t]_{1 + (t-1 \mod n)} = c_\alpha^n \cdot [\xv_{t-1}]_{1 + (t-1 \mod n)}\,. \label{eq-cycling}
\end{align}
Observe $\nabla f(\xv_0) = Q\xv_0$. Hence the GS rule picks $i_1=1$ in the first iteration. The iterate is updated as follows:
\begin{align}
 [\xv_{1}]_1 &\refEQ{eq:grad_update}[\xv_{0}]_1  - \frac{[Q\xv_0]_1}{Q_{11}} \\
    & = 1 - \frac{(\alpha -1)\frac{1}{n} S_{n}(c_\alpha) + 1}{(\alpha -1)\frac{1}{n} + 1} \\
    & = \frac{(\alpha -1)\frac{1}{n} \left(1 - S_{n}(c_\alpha) \right)}{(\alpha -1)\frac{1}{n} + 1} \\
    & = \frac{(\alpha -1)\frac{1}{n} \left(c_\alpha^n - c_\alpha S_{n}(c_\alpha) \right)}{(\alpha -1)\frac{1}{n} + 1}  \\
    & \refEQ{eq-defC} \frac{(\alpha -1)\frac{1}{n} c_\alpha^n  + c_\alpha^n}{(\alpha -1)\frac{1}{n} + 1} = c_\alpha^n
\end{align}
The relation~\eqref{eq-cycling} can now easily be checked by the same reasoning and induction. 

It remains to verify that for this sequence property~\eqref{eq-lowerboundratio2} holds. This is done in Lemma~\ref{lemma-lowbound2}. Note that $\nabla f(\xv_0) = Q \xv_0 = \gv$, where $\gv$ is defined as in the lemma, and that all gradients $\nabla f(\xv_t)$ are up to scaling and reordering of the coordinates equivalent to the vector $\gv$.
\end{proof}

\begin{lemma}
\label{lemma-lowbound1}
Let $0 < \alpha < \frac{1}{2}$ and $0 < c_\alpha <1$ defined by  equation~\eqref{eq-defC}, where $S_n(c_\alpha) = \sum_{i=0}^{n-1} c_\alpha^n$. Then $c_\alpha \geq 1 - \frac{4}{n}\alpha$ for $\alpha \in [0, \frac{1}{2}]$.
\end{lemma}
\begin{proof}
Using the summation formula for geometric series, $S_n(c_\alpha) = \frac{1 - c_\alpha^n}{1-c_\alpha}$ we derive
\begin{align}
 \alpha \refEQ{eq-defC} 1 - \frac{n c_\alpha^{n-1}}{S_n(c_\alpha)} = \underbrace{1 - \frac{n (1- c_\alpha) c_\alpha^{n-1}}{1 - c_\alpha^{n}}}_{:= \Psi(c_\alpha)} \,. \label{eq-cA}
\end{align} 
With Taylor expansion we observe that
\begin{align}
\Psi \left( 1 - \frac{3\alpha}{n} \right) &\geq \alpha\,, &   \Psi  \left( 1 - \frac{2\alpha}{n} \right) &\leq \alpha
\end{align}
where the first inequality only hold for $n > 2$ and $ \alpha \leq \in [0, \frac{1}{2}]$. Hence any solution to~\eqref{eq-cA} must satisfy $c_\alpha \geq 1 - \frac{3}{n}\alpha$.
\end{proof}

\begin{lemma}
\label{lemma-lowbound2}
Let $c_\alpha$ as in~\eqref{eq-defC}. Let $\gv \in \R^n$ be defined as
\begin{align}
[\gv]_i =\frac{ (\alpha - 1) \frac{1}{n} S_n (c_\alpha) + c_\alpha^{i-1}}{ 1 + \frac{ \alpha-1}{n}  }
\end{align} Then 
\begin{align}
\max_{i \in [n]} \frac{\norm{\gv}_\infty^2}{\frac{1}{n}\norm{\gv}_2^{2}} \leq 3 + 3\alpha \,.\
\end{align}
\end{lemma}
\begin{proof}
Observe
\begin{align}
 [\gv]_i &= \frac{ (\alpha - 1) \frac{1}{n} \left( S_{n-1} (c_\alpha) + c_\alpha^{n-1} \right)+ c_\alpha^{n-1} + (c_\alpha^{i-1} - c_\alpha^{n-1})}{ 1 + \frac{ \alpha-1}{n}  } \\ 
 &\refEQ{eq-defC}  \frac{c_\alpha^{i-1} - c_\alpha^{n-1}}{ 1 + \frac{ \alpha-1}{n}  }
\end{align}
Thus $[\gv]_1 > [\gv]_2 > \dots > [\gv]_n$ and the maximum is attained at
\begin{align}
\omega(\gv) := \frac{[\gv]_1^2}{\frac{1}{n} \sum_{i=1}^n [\gv]_i^2} &= 
 \frac{c_\alpha^{2} \left(c_\alpha^{2}-1 \right)  \left(1-c_\alpha^{n-1} \right)^2 n}{2 c_\alpha^{n+1}  + 2 c_\alpha^{n + 2} -2 c_\alpha^{2n + 1} + (n-1) c_\alpha^{2n+2} -c_\alpha^{2} -nc_\alpha^{2n} }
\end{align}
For $c_\alpha \geq 1 - \frac{3}{n}\alpha$  and $\alpha \leq \frac{1}{2}$, this latter expression can be estimated as
\begin{align}
\omega(\gv) \leq 3 + 3\alpha
\end{align}
especially $\omega(\gv) \leq 4$ for $\alpha \leq \frac{1}{3}$.
\end{proof}

\section{Approximate Gradient Update} \label{appen:formulation}
In this section we will prove Lemma~\ref{lem:gen_grad_update}. Consider first the following simpler case, where we assume $f$ is given as in least squares, i.e. $f(\xv) := \frac{1}{2} \norm{A \xv - \bv}^2$. 

In the $t_{th}$ iteration, we choose coordinate $i_t$ to optimize upon and the update from $\xv_{t+1}$ to $\xv_t$ can be written as $\xv_{t+1} = \xv_t + \gamma_t \evv_{i_t}$. Now for any coordinate $i$ other than $i_t$, it is fairly easy to compute the change in the gradient of the other coordinates. We already observed that $[\xv_t]_j$ does not change, hence the sub-gradient set of $\Psi_{j}([\xv_t]_j)$ and $\Psi_{j}([\xv_{t+1}]_j)$  are equal. For the change in $\nabla f$, consider the analysis below:
\begin{align}
\nabla_i F(\xv_{t+1}) - \nabla_i F(\xv_{t}) &= \av_i ^\top (A \xv_{t+1} - b) - \av_i ^\top (A \xv_t - b)  \\
&=  \av_i ^\top \big(A( \xv_{t+1} - \xv_{t} )  \big )   \label{eq:31part1}\\
&= \av_i ^\top \big(A(\xv_{t} + \gamma_t \evv_{i_t} - \xv_{t} )  \big ) \label{eq:31part2} \\
&= \av_i ^\top \big(\gamma A \evv_{i_t}  \big ) = \gamma_t \av_i ^\top \av_{i_t}
\end{align}
Equation \eqref{eq:31part2} comes from the update of $\xv_t$ to $\xv_{t+1}$.

By the same reasoning, we can now derive the general proof.

\begin{proof}[\textbf{Proof of Lemma~}\ref{lem:gen_grad_update}]
Consider a composite function $F$ as given in Lemma~\ref{lem:gen_grad_update}. By the same reasoning as above, the two sub-gradient sets of $\Psi_{j}([\xv_t]_j)$ and $\Psi_{j}([\xv_{t+1}]_j)$  are identical, for every passive coordinate $j \neq i_t$. The gradient of $F$ can be written as:
\[
\nabla_i F(\alphav_t) = \av_i ^\top \nabla f(A\alphav_t)
\]
For any arbitrary passive coordinate $j \neq i_t$ the change of the gradient can be computed as follows: 
\begin{align}
 \nabla_j F(\xv_{t+1}) - \nabla_j F(\xv_t)  
&=  \av_j ^\top \nabla f(A \xv_{t+1})  -  \av_j ^\top \nabla f(A \xv_t) \notag  \\
&=  \av_j ^\top  ( \nabla f(A \xv_{t+1})  -\nabla f(A \xv_t) ) \label{eq:32part1} \\
&=  \av_j ^\top \Big ( \nabla f\big (A(\xv_{t} + \gamma_t \evv_{i_t})\big)  -\nabla f\big(A \xv_t\big) \Big ) \notag \\ 
&\stackrel{\ast}{=} \big \langle A^\top \nabla^2 f(A \tilde{\xv}) \av_j, \xv_{t+1} - \xv_{t} \big \rangle  \notag \\
& =  \big \langle \gamma_t  \nabla^2 f(A \tilde{\xv}) \av_j , A ( \xv_{t+1} - \xv_{t}) \big \rangle \notag \\
&= \gamma_t \av_j^\top  \nabla^2 f(A\tilde{\xv}) \av_{i_t}  \label{eq:approx_general-again}
\end{align}
Here $\tilde{\xv}$ is a point on the line segment between $[\xv_t]_{i_t}$ and $[\xv_{t+1}]_{i_t}$ which can be found by the Mean Value Theorem.
\end{proof}

\section{Algorithm and Stability}
\begin{proof}[\textbf{Proof of Theorem}~\ref{thm-competitiveRatio2}]
As we are interested to study the expected competitive ration $\E{\rho_t}$ for $t \to \infty$, we can assume mixing and consider only the steady state. 

Define $\alpha_t \in [0,1]$ s.t. $\alpha_t(n-s) = \abs{\{i \in \mathcal{I}_t \mid i > s\}}$. I.e. $\alpha_t(n-s)$ denotes the number of indices in $\abs{\mathcal{I}_t}$ which do not belong to the set $[s]$.

Denote $\alpha_\infty := \lim_{t \to \infty} \alpha_t$. 
By equilibrium considerations, the probability that an index $i \notin [s]$ gets picked (and removed from the active set), i.e. $1 - \rho_\infty$, must be equal to the probability that an index $j \notin [s]$ enters the active set. Hence
\begin{align}
 \frac{(1-\alpha_\infty)(n-s)}{T_\infty} = 1 - \rho_\infty = \frac{\alpha_\infty (n-s)}{\alpha_\infty (n-s) +c s }\,.
\label{eq-equilib}
\end{align}
We deduce the quadratic relation $ \alpha_\infty T_\infty = (1- \alpha_\infty) \left(  \alpha_\infty  (n-s) + cs\right)$ with solution
\begin{align}
 \alpha_\infty = \frac{n-(1+c)s - T_\infty + \sqrt{n^2 + (c-1)^2s^2 + 2n((c-1)s - T_\infty) + 2(1+c)sT_\infty +T_\infty^2}}{2(n-s)}\,.
 \label{eq-solveA}
\end{align}
Denote $\theta := n^2 + (c-1)^2s^2 + 2n((c-1)s - T_\infty) + 2(1+c)sT +T_\infty^2$. Hence,
\begin{align}
\rho_\infty \refEQ{eq-equilib} \frac{cs}{\alpha_\infty (n-s) + cs} \refEQ{eq-solveA} \frac{2 c s}{cs + n -s - T_\infty + \sqrt{\theta}}\,.
\end{align}
We now verify the provided lower bound on $\rho_\infty$:
\begin{align}
\rho_\infty \refEQ{eq-equilib}  1 - \frac{(1-\alpha_\infty)(n-s)}{T_\infty} \geq 1 - \frac{n-s}{T_\infty}\,.
\end{align}
This bound is sharp for large values of $T_\infty$, ($T_\infty > 2n$, say), but trivial for $T_\infty \leq n-s$.
\end{proof}

\section{GS rule for Composite Functions}
\subsection{GS-q rule}
\label{sec:GSq}
In this section we show how ASCD can be implemented for the GS-q rule.
Define the coordinate-wise model
\begin{align}
 V_i(\xv,y,s) := sy + \frac{L}{2}y^2 + \Psi_i(x_i+y)
\end{align}
The GS-q rule is defined as (cf.~\citet{nutini2015coordinate})
\begin{align}
 i = \argmin_{i \in [n]} \min_{y \in \R} V(\xv,y, \nabla_i f(\xv))
\end{align}

First we show that the vectors $\vv$ and $\wv$ defined in Algorithm~\ref{alg-4} gives valid upper and lower bounds on the value of $\min_{y \in \R} V(\xv,y, \nabla_i f(\xv))$. We start with the lower bound $\vv$:

Suppose we have upper and lower bounds, $\ell \leq \nabla_i f(\xv) \leq u$ on one component of the gradient. 
Define $\alpha \in [0,1]$ such that  $\nabla_i f(\xv) = (1-\alpha) \ell + \alpha u$. Note that
 \begin{align}
 (1-\alpha) V_i(\xv,y,\ell) + \alpha V_i(\xv,y,u) = V_i(\xv,y,\nabla_i f(\xv))
 \end{align}
Hence,
\begin{align}
\min \left\{ \min_{y} V_i(\xv,y,u),  \min_{y} V_i(\xv,y,\ell) \right\} \leq \min_y V_i(\xv,y,\nabla_i f(\xv)) \,.
\end{align}
The derivation of the upper bounds $\wv$ is a bit more cumbersome. 
Define $\ell^\star := \argmin_{y\in\R} V_i(\xv,y,\ell)$, $u^\star := \argmin_{y\in\R} V_i(\xv,y,u)$ and observe:
\begin{align}
V_i(\xv,u^\star,\nabla_i f(\xv)) &= V_i(\xv,u^\star,u) - (u - \nabla_i f(\xv)) u^\star \leq  V_i(\xv,u^\star,u) - uu^\star + \max\{uu^\star, \ell u^\star\} =: \omega_u  \\
 V_i(\xv,\ell^\star,\nabla_i f(\xv)) &= V_i(\xv,\ell^\star,\ell) - (\ell - \nabla_i f(\xv)) \ell^\star \leq V_i(\xv,\ell^\star,\ell) - \ell \ell^\star + \max\{u \ell^\star, \ell \ell^\star\} =: \omega_\ell \\
 V_i(\xv,0,\nabla_i f(\xv)) &= \Psi_i([\xv]_i)
\end{align}
Hence $\min_y V_i(\xv,y,\nabla_i f(\xv)) \leq \min \{ \omega_\ell, \omega_u, \Psi_i([\xv]_i) \}$. 

Note 
\begin{align}
\omega_u &=  V_i(\xv,u^\star,u) + \max\{0, (\ell-u) u^\star\} \\
\omega_\ell &=  V_i(\xv,\ell^\star,\ell)  + \max\{0, (u-\ell) \ell^\star\}
\end{align}
which coincides with the formulas in Algorithm~\ref{alg-4}.

It remains to show that the computation of the active set is save, i.e. that the progress achieved by ASCD as defined in Algorithm~\ref{alg-4} is always better than the progress achieved by UCD. Let $\mathcal{I}$ be defined as in Algorithm~\ref{alg-4}. Then
\begin{align}
 \frac{1}{\abs{\mathcal{I}}}  \sum_{i \in \mathcal{I} }   \min_{y \in \R} V_i(\xv,y,\nabla_i f(\xv)) &\leq  \frac{1}{n}  \sum_{i \in [n]}   \min_{y \in \R} V_i(\xv,y,\nabla_i f(\xv)) \\
 & = \frac{1}{n} \min_{\yv \in \R^n}  \sum_{i \in [n]}  V_i(\xv,y,\nabla_i f(\xv)) \,.
\end{align}
Using this observation, and the same lines of reasoning as given in \citep[Section H.3]{nnmf}, it follows immediately that the one step progress of ASCD is at least as good as the for UCD. 

\subsection{GS-r rule}
\label{sec:GS-r}
With the notation $[\yv^\star]_i := \argmin_{y \in \R} V_i(\xv,y,\nabla_i f(\xv))$, the GS-r rule is defined as (cf.~\citet{nnmf})
\begin{align}
 i = \argmax_{i \in [n]} \abs{[\yv^\star]_i}\,.
\end{align}
In order to implement ASCD for GS-r, we need therefore to maintain lower and upper bounds on the values $\abs{[\yv^\star]_i}$.

Suppose we have upper and lower bounds, $\ell \leq \nabla_i f(\xv) \leq u$ on one component of the gradient. Define $\ell^\star := \argmin_{y\in\R} V_i(\xv,y,\ell)$, $u^\star := \argmin_{y\in\R} V_i(\xv,y,u)$, then $y^\star$ is contained in the line segment between $\ell^\star$ and $u^\star$. 
Hence as in Algorithm~\ref{alg-1}, the lower and upper bounds can be defined as
\begin{align}
[\uv_t]_i &:= \max_{y \in \R } \{\ell^\star \leq y \leq u^\star \}\\
[\ellv_t]_i &:= \min_{y \in \R } \{\ell^\star \leq y \leq u^\star \}
\end{align}

However, note that in~\cite{nutini2015coordinate} it is established that GS-r rule can be worse than UCD in general. Hence we cannot expect that ASCD for the GS-r rule is better than UCD in general. However, the by the choice of the active set, the index chosen by the GS-r rule is always contained in the active set, and ASCD approaches GS-r for small errors.

\section{Experimental Details }
\label{sec:datageneration}
We generate a matrix $A \in \mathbb{R}^{m \times n}$ from the standard normal $\mathcal{N}(0,1)$ distribution. $m$ is kept fixed at $1000$ but $n$ is chosen $1000$ for the $l_2$ regularized least squares regression and $5000$ for $l_1$ regularized counterpart. $1$ is added to each entry (to induce a dependency between columns),  multiplied each column by a sample from   $\mathcal{N}(0,1)$ multiplied by ten (to induce different Lipschitz constants across the coordinates), and only kept each entry of $A$ non-zero with probability $10 \frac{\log(n)}{n}$. This is exactly the same procedure which has been discussed in \cite{nutini2015coordinate}.

\end{document}